\documentclass[12pt]{iopart}

\usepackage{iopams}  
\expandafter\let\csname equation*\endcsname\relax
\expandafter\let\csname endequation*\endcsname\relax

\usepackage{amsmath}

\usepackage{amsfonts}
\usepackage{float}
\usepackage{dsfont}

\usepackage{subfigure} 
\usepackage{graphicx} 

\usepackage[normalem]{ulem}
\usepackage{comment}

\newcommand{\RR}{\mathds{R}}

\newcommand{\dsum}{\displaystyle\sum}

\newcommand{\Rad}[1]{\mathcal{R}\left( #1 \right)}

\newcommand{\D}[1]{\,\mbox{d} #1} 

\newcommand{\Xiom}{\xi_\omega}
\newcommand{\XiomT}{(\cos\omega,\sin\omega)}

\newcommand{\JC}{\mathcal{J}}
\newcommand{\JCR}{\mathcal{J}_R}
\newcommand{\DC}{\mathcal{D}}
\newcommand{\SC}{\mathcal{S}}
\newcommand{\RC}{\mathcal{R}}
\newcommand{\VC}{\mathcal{V}}

\usepackage{amsthm}
\newtheorem{theorem}{Theorem}
\newtheorem{definition}{Definition}
\newtheorem{lemma}{Lemma}

\DeclareMathOperator{\trace}{trace}
\DeclareMathOperator{\jac}{D}
\DeclareMathOperator{\hess}{D^2}

\begin{document}

\title[Methods based on Radon transform for non-affine DIR of noisy images]{ Methods based on Radon transform for non-affine deformable image registration of noisy images}

\author{
	Daniel E. Hurtado$^{1,2}$, Axel Osses$^{3,4}$ and Rodrigo Quezada$^{3,5,\text{\footnote{Corresponding author}}}$
 }
\address{$^1$ Department of Structural and Geotechnical Engineering, School of Engineering, Pontificia Universidad Católica de Chile, Santiago, Chile.} 
\address{$^2$ Institute for Biological and Medical Engineering, Pontificia Universidad Católica de Chile, Santiago, Chile.}
\address{$^3$ Departamento de Ingeniería Matemática, Universidad de Chile, Santiago, Chile. }
\address{$^4$ Centro de Modelamiento Matemático, UMI 2807, Universidad de Chile-CNRS, Santiago, Chile.}
\address{$^5$ Departamento de Producción Vegetal, Facultad de Agronomía, Universidad de Concepción, Chillán, Chile. 
}

\ead{rodrigoaquezada@udec.cl
}

\begin{quote}
    
\end{quote}
\vspace{10pt}
\begin{indented}
\item[] August 2024 
\end{indented}

\begin{abstract}
Deformable image registration is a standard engineering problem used to determine the distortion experienced by a body by comparing two images of it in different states. This study introduces two new DIR methods designed to capture non-affine deformations using Radon transform-based similarity measures and a classical regularizer based on linear elastic deformation energy. It establishes conditions for the existence and uniqueness of solutions for both methods and presents synthetic experimental results comparing them with a standard method based on the sum of squared differences similarity measure. These methods have been tested to capture various non-affine deformations in images, both with and without noise, and their convergence rates have been analyzed. Furthermore, the effectiveness of these methods was also evaluated in a lung image registration scenario.
\end{abstract}

%
\vspace{2pc}
\noindent{\it Keywords}: Deformable image registration, DIR, non-affine registration, similarity measure, Radon transform, lung registration.
%
%
%
%

\section{Introduction}
\label{sec:Introduction}

The Deformable Image Registration (DIR) problem is a mathematical problem that allows the computation of the deformation field experienced by a body using two images in different states. Given two images, R (reference) and T (template), of a body before and after deformation, and a deformation model with desirable mechanical properties,
the DIR problem can be defined as follows: 
\begin{eqnarray}\fl \qquad
\mbox{\it Find the deformation } u:\Omega\subset\RR^2\to\RR^2 \mbox{ \it such that } T_u:=T\circ \{id+u\} \simeq R,
\end{eqnarray}
where $\Omega$ is the image domain and $id$ corresponds to the identity function. \nocite{ruthotto2015non}

The DIR problem has been extensively utilized in the field of biomedicine for a multitude of applications, including anatomy segmentation, assisted surgery, and image fusion. These applications have been employed to analyze the diverse structures of the human body, such as different organs \cite{oliveira2014medical}.
Especially in biomedicine, the DIR problem has been widely used for applications such as anatomy segmentation, assisted surgery, and image fusion, applying them to the different human organs \cite{oliveira2014medical}. 
For instance, the DIR problem has gained significant importance in recent decades due to the rising prevalence of lung diseases associated with lung stiffness. This task has presented a significant challenge due to the large and non-affine deformations that this organ exhibits.
Consequently, it is imperative to develop novel DIR methods to compute large non-affine deformations, which will facilitate further technological advancement in the early detection and treatment of lung diseases 
\cite{concha2018micromechanical,Cruces2019mapping,hurtado2017spatial,retamal2018does}.

The literature highlights a variational approach to the DIR problem, which consists of finding the deformation field $u$ by minimizing the functional given by Equation \eqref{eq:DIR Problem},
where $\mathcal{S}$ is a measure of the similarity between the images $R$ and $T$, $\mathcal{R}$ is a regularizer, and $\alpha$ is the corresponding regularization parameter
\cite{aubert2006mathematical, vese2016variational, oektem2017shape, barnafi2018primal, chen2018indirect}. 
The use of intensity-based similarity measures has become widespread in recent research due to their simplicity, speed, and the fact that they do not require segmentation
\cite{maintz1998survey,oliveira2014medical}.
Of particular interest is the sum of squared differences (SSD) similarity measure, which is defined as the $L^2$ norm of the difference between images and is straightforward to implement.
To prevent the acquisition of suboptimal local solutions, it is imperative to incorporate a robust regularizer into the model.
A regularization based on elastic energy in the DIR model, such as Linear Elastic Energy (LEE), is essential for obtaining physically consistent displacements
\cite{BAJCSY19891}. 
The method using the SSD similarity measure and the LEE regularizer has been utilized in numerous studies to solve the DIR problem
\cite{ashburner1999nonlinear,brown1992survey, capek1999optimisation,modersitzki2004numerical}. 
In practice, this approach works well for capturing small and noise-free registrations. However, it is less suitable for large deformations or images with significant noise  \cite{BAJCSY19891}.

Registration between noisy images presents a significant challenge for intensity-based similarity measures. The presence of noise can generate false matches or affect gradient computation in optimization methods, potentially leading to amplified values
\cite{BAJCSY19891,cain2001projection}. 
Consequently, a variety of registration methods have incorporated techniques to mitigate the effects of noise, including projection-based similarity measures which serve to average out the noise in images \cite{albu2014transformed,BAJCSY19891,cain2001projection,khamene2006novel,sauer1996efficient}.
Some two-dimensional DIR works have incorporated projections along the rows and columns of the images in order to reduce the effects of noise in the registration process \cite{albu2014transformed,sauer1996efficient}. Others have reported similar effects over the noise when the Radon transform is incorporated into the DIR problem formulation for computing linear displacements \cite{mooser2013estimation,mooser2009estimation,nacereddine2015similarity}.
In the context of affine deformation, several studies have proposed a direct registration approach based on X-ray measurements using the linearity properties of  the Radon transform, by directly comparing the sinograms of $R$ and $T_u$ \cite{JiangshengYou1998,Mao2007a,mooser2009estimation,mooser2013estimation,yan2005edge}.
A review of the literature reveals no evidence that similar techniques have been employed for the registration of non-affine deformations.

This paper presents two new Radon-based similarity measures that are expected to be useful for solving the DIR problem.
The first method is computed in the Radon space by comparing the Radon transforms of the images. The second method is computed in the image domain, but by comparing the backward projection of the Radon transforms of the images instead of the images themselves. The existence and uniqueness of the solutions of both proposed methods are analyzed. Additionally, the performance of both similarity measures for registering noisy images or images with large deformations, including a lung case between noisy images, has been included.

This paper is organized as follows:
In Section 2, a review of the mathematical settings necessary to formulate the DIR problem is presented.
The proposed DIR methods are described in Section 3, together with an analysis of the existence and uniqueness conditions of the solutions to the DIR problem in Section 4. Section 5 covers the computational considerations necessary to implement the method. Numerical tests are presented in Section 6 using several examples. Finally, Section 7 discusses the strengths, scope, and limitations of the methods, as well as possible extensions of this work.

\section{DIR Problem}
\label{sec:DIRProblem}

In a general, the DIR problem can be defined as the task of identifying an optimal geometric transformation, denoted by $u$, between two images, $R$ and $T$, in different states of a given object. This transformation must satisfy the following condition: $R(x)\approx T(x+u(x))$, for all $x$ in the image domain.

In the variational framework, a regularized version of the DIR problem is given by: 
Given two images $R:\Omega\to\RR$, $T:\tilde\Omega\to\RR$, both in $H^1$,
\textit{find $u:\Omega\to\RR^2$ such that:}
\begin{equation}
\JC(u):=\DC(T,R;u)+\alpha\SC(u)\,\stackrel{u}{\to}\,\min , 
\label{eq:DIR Problem} 
\end{equation}
\textit{where $\Omega\subseteq\tilde\Omega \subseteq \RR^2$ are  image domains,  $\DC$ is a similarity measure depending on  $R,$ $T,$ and $u$; and $\SC$ is a regularization term  with regularization parameter $\alpha>0$}   \cite{modersitzki2004numerical}.

It is important to note that the domain of $T$ must be larger than the domain of $R$, because when the domain of $T$ is deformed, there are points of $\Omega$ that could be transformed out of $\Omega.$
In practice, when the DIR problem is solved, the same image domain $\Omega$ is considered, where the image $T_u:\tilde\Omega|_\Omega\to\RR$ is restricted to $\Omega.$.
This implies that if $u$ is replaced by $u_M$, where 
$$u_M(x)=
\begin{cases}
u(x),&\quad |u(x)|<M\\
0,&\quad |u(x)|\geq M,
\end{cases}$$
for $M$ sufficiently large, then $T_u=T_{u_M}.$ Then it makes sense to assume that $u$ is bounded.

\begin{figure}
	\centering
	\includegraphics[width=0.3\linewidth]{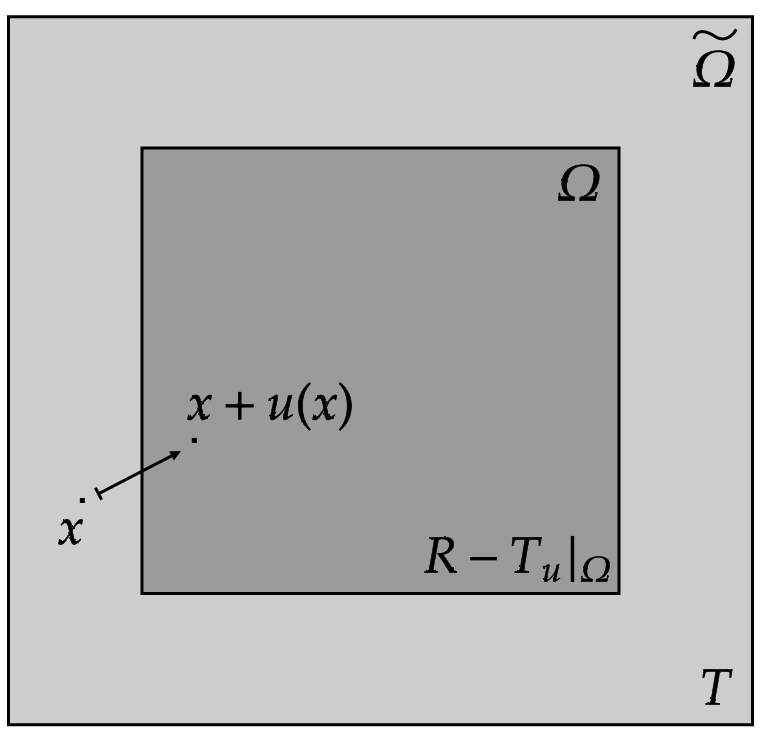}
	\caption{
		The figure illustrates the $\Omega$ and $\tilde\Omega$ domains of the images $R$ and $T,$ respectively. In addition, it is shown how a point $x \in \Omega$ could be transported by the vector field $u$ into the $\Omega$ domain.
 }
	\label{fig:diagram}
\end{figure}

Due to its mathematical simplicity of definition and implementation, SSD is a popular choice to solve the DIR problem. The definition of SSD is as follows:
\begin{align}
\DC_{SSD}(R,T;u) &:=\dfrac{1}{2}\int_{\Omega } \big(\,T_u(x)-R(x) \,\big)^2\D{x}, \label{eq:DSSD}
\end{align}
where $R,~T$ and $u$ are the same as before.
Similarly, the LEE regularizer is a typical choice of regularizer to solve the DIR problem. It has desirable physical properties that contribute to the solution of the DIR problem. It is defined as follows:
\begin{equation}
\SC(u):=\dfrac{1}{2}\int_\Omega \mathds{C}\,\varepsilon(u):\varepsilon(u) \D{x},\label{eq:LE} 
\end{equation} 
where $\varepsilon(u):=\frac{1}{2}\left(\nabla u +\nabla u^t\right)$ is the infinitesimal strain tensor  and  $\mathds{C}$ is the elasticity tensor, that for an isotropic solid is defined by 
\begin{equation}
\mathds{C}\,\varepsilon  := \lambda\, tr( \varepsilon ) \,\textbf{I} +2\, \mu\, \varepsilon,
\label{eq:ElasticityTensor}
\end{equation}
 where $\lambda,$ $\mu$ are the  constants of Lame.

The DIR method, employing the SSD similarity measure and linear elastic energy, has been examined in \cite{barnafi2018primal}. It was demonstrated that a solution exists under specific reasonable assumptions about the images.


\section{Similarity measures based on the Radon transform}
\label{sec:ProposedMethod}

The present paper introduces two variational techniques for the resolution of the DIR Problem \eqref{eq:DIR Problem} using two novel similarity measures in conjunction with the well-established LEE regularizer.

In order to introduce the proposed methods, it will be necessary to present some tools which will allow the definition of new similarity measures to solve the DIR problem. These tools include the Radon transform, its adjoint transform, and the sinogram.
Let $\Omega=(-a,a)^2$, with $a>0$. Let $f\in L^2(\Omega)$ be an image and $g\in L^2(S^1\times [-a,a])$. The Radon transform of $f$ and the adjoint of the Radon transform of $g$, respectively, 
are defined as continuous linear operators as follows:
\[
\mathcal{R}: L^2(\Omega) \to L^2(S^1\times\RR)
\quad \mbox{ and } \quad
\mathcal{R}^\#:L^2(S^1\times \RR )\to L^2(\RR^2).
\]

The\textit{ Radon transform of} $f$ is defined as $\Rad f :S^1\times \RR \to \RR$, where:
\begin{equation}
\Rad f [\Xiom,s] = \int_{ \Omega} f(x)\,\delta_0(x\cdot \Xiom -s) \,\mbox{d}x, \label{eq:RadonTransform}
\end{equation}
where $S^1=\{\Xiom = \XiomT | \omega \in (0,\pi) \},$ and $\delta_0$ is the Dirac delta at the origin.
In the same way, the adjoint of the Radon transform is $\mathcal{R}^\#(g)[x]: \Omega\to\RR$ where:
\[\mathcal{R}^\#(g)[x]:=\int_{S^1} g (\theta,x\cdot\theta) \D{\theta}.\]
 $\mathcal{R}^\#$ is also called back-projection operator.
In the two-dimensional setting, a geometric interpretation of the Radon transform of the function $f$ with arguments $[\Xiom,s]$ corresponds to the integral along a straight line that is located at a distance $s$ from the origin of the coordinate system over the image domain and at an angle $\Xiom^\perp$. The adjoint of the Radon transform is interpreted as the sum of all line integrals passing through the point $x$.
Furthermore, the term \textit{sinogram} is introduced, which corresponds to a discrete visual representation of the Radon transform as an image, organized in a matrix arrangement according to the angle $\theta$ and the distance $s$. 

This study proposes two new similarity measures called $\DC_{RSSD}$ and $\DC_{R\#SSD},$ defined by:
\begin{align}
\DC_{RSSD}(R,T;u )&:=\dfrac{1}{2}\int_{S^1\times \RR } \left(\,\RC(T_{u})[\Xiom,s]-\RC(R) [\Xiom,s]\,\right)^2\D{\Xiom\D{s}}, \label{eq:DFR}
\end{align}
where   $\RC$ denotes the Radon transform and $u,R,T$ are the same as before; and
$\DC_{RSSD}$ and $\DC_{R\#SSD},$ defined by:
\begin{align}
	\DC_{R^\#RSSD}(R,T;u )&:=\dfrac{1}{2}\int_{\Omega} \left(\,\RC^\#(\RC(T_{u}(x)))-\RC^\#(\RC(R(x))) \,\right)^2\D{x}. \label{eq:DFRG}
\end{align}
It is important to note that the Radon transform of an image, denoted by $\RC(I)$, can be understood computationally as the sinogram of the image $I$.
Similarly, $\RC^\#(\RC(I))$ can be understood as the back-projection of the sinogram $\RC(I)$, also called the pseudoinverse of $I.$
However, it is crucial to emphasise that $I$ and $\RC^\#(\RC(I))$ are both images of the same dimension, but they are not the same image.

In this way, the proposed DIR problems are as follows:  \newline
\textit{Find $u\in H^1(\Omega)$ such that:}
\begin{equation}
\JC_{RSSD}(u):=\DC_{RSSD}(R,T;u)+\alpha_{RSSD} \SC(u) \,\to\,\min,
\label{eq:ProposedDIRProblem}
\end{equation}
and,
\begin{equation}
	\JC_{R^\#RSSD}(u):=\DC_{R^\#RSSD}(R,T;u)+\alpha_{R^\#RSSD} \SC(u) \,\to\,\min
	\label{eq:ProposedDIRProblem2}
\end{equation}
where $\SC$ is the LEE regularizer \eqref{eq:LE}, $\alpha_{RSSD}$ and $\alpha_{R^\#RSSD}$ are  regularization constants.


\section{Conditions for the existence and uniqueness of the solution}
\label{sec:Existence}

In the present section, certain properties are demonstrated in order to determine the conditions for the existence and uniqueness of the solution to the continuous DIR problems given in equations \eqref{eq:ProposedDIRProblem} and \eqref{eq:ProposedDIRProblem2}. 
These conditions are a direct consequence of the result presented by Barnafi \textit{et.al.} (2018) in \cite{barnafi2018primal}. 
The authors employed the Euler-Lagrange equations to derive a weak formulation for the DIR problem, analogous to that in Equation \eqref{eq:DIR Problem}, which enabled them to determine the existence and uniqueness of solutions through the application of classical functional analysis techniques and under small data assumptions.

Before starting, it is necessary to define the functional space that will be employed.
\begin{definition}
Let us define the set $H$ as follows: $H:=\{u\in H^1(\Omega):\varepsilon(u)=0\}^\perp$, with $\Omega\subset \RR^2$.
\label{def:H}
\end{definition}

The following lines present a summary of the theorem presented by Barnafi \textit{et.al.} (2018). 
\begin{theorem}[Barnafi et al. (2018) \cite{barnafi2018primal}] \label{thm:Barnafi}
Let $\,\mathcal T: H \rightarrow  H$ be an operator with $\mathcal T(z) = u$, where for each $z\in H$, $u$ is the solution of the problem:\\
\textit{Find $u \in  H$ such that}
\begin{equation}\label{eq:WeakDIRProblemNoUnique}
	a(u,v)=\hat\alpha F_z(v), \qquad\forall v\in H,
 \end{equation}
where $\hat\alpha$ is the regularization parameter associated to the DIR problem, $F_z\in H',$ and $a$ is the continuous and non-negative bilinear form in $H^1(\Omega)$, given by
\begin{equation}
	a(u,v):=\int_\Omega \mathbb{C}\varepsilon(u):\varepsilon(v) \D{x}, \label{eq:regularizator}
\end{equation}
where $\mathbb C $ was defined in \eqref{eq:ElasticityTensor}.
	
Under the following data assumptions:
\begin{description}
		\item[H1] $F_u$ is Lipschitz with  respect to $u\in H$, \textit{i.e.},
		$\|F_u-F_v\|_{H^\prime}\leq L_F\|u-v\|_{H}, \forall u,v\in H,$
		\item[H2] $F_u$ is bounded in $H'$, \textit{i.e.},
		$\|F_u\|_{H^\prime}\leq M_F,\quad \forall u\in H,$
	\end{description}
then the operator $\mathcal T$ has at least one fixed point. 
Furthermore, if $\hat\alpha  C L_F < 1$, where $C$ is a constant coming from an \textit{a priori} estimation given by \cite[Theorem 2]{barnafi2018primal}, the fixed point is unique.
	\label{th:Barnafi2018}
\end{theorem} 

The proof of this theorem employs a variety of tools, including Schauder's fixed point theorem, inequalities derived from the Lax-Milgram lemma and Korn's inequality, along with other standard results from functional analysis, under small data assumptions. For further details the reader can consult the reference \cite{barnafi2018primal}.

It is important to note that the Problems \eqref{eq:ProposedDIRProblem} and \eqref{eq:ProposedDIRProblem2} can be redefined according with the formulation presented in \cite{barnafi2018primal}. For this, first notice that
the aforementioned Problems \eqref{eq:ProposedDIRProblem} and \eqref{eq:ProposedDIRProblem2} can be redefined  as follows:

\textit{Find $u\in H^1(\Omega),$ such that:}
\begin{equation}
\hat\JC_{RSSD}(u):=\SC(u)+\hat\alpha_{RSSD}\DC_{RSSD}(R,T;u)\,\to\,\min,
\label{eq:PropsedDIRProblemAdapted}
\end{equation}
where $\hat\alpha_{RSSD}:=\frac{1}{\alpha_{RSSD}}$. And,

\textit{Find $u\in H^1(\Omega),$ such that:}
\begin{equation}
\hat\JC_{R^\#RSSD}(u):=\SC(u)+\hat\alpha_{R^\#RSSD}\DC_{R^\#RSSD}(R,T;u)\,\to\,\min,
\label{eq:PropsedDIRProblemAdapted2}
\end{equation}
where $\hat\alpha_{R^\#RSSD}:=\frac{1}{\alpha_{R^\#RSSD}}$.

This is easy to see by dividing the functional $\JC_{RSSD}$, as defined in Problem \eqref{eq:ProposedDIRProblem}, by the constant $\alpha_{RSSD}$, and defining  $\hat\alpha_{RSSD}:=\frac{1}{\alpha_{RSSD}}$, the desired result is obtained.
The equivalence for Problem \eqref{eq:ProposedDIRProblem2} is analogous.

To study the existence and uniqueness of the solutions of the proposed Problems \eqref{eq:PropsedDIRProblemAdapted}  and \eqref{eq:PropsedDIRProblemAdapted2}, the existence and uniqueness of fixed points of the operators associated to these problems are analyzed by verifying the hypotheses {$\mathbf{H1}$} and {$\mathbf{H2}$} of the Theorem \ref{th:Barnafi2018}, and consequently it is verified that these fixed points are indeed local minima of the given problems.

Regarding the images and the solution space, let $\Omega = (-a,a)^2\subset\RR^2$ be the image domain with center $O=(0,0)$, $H^1(\Omega)$ be a reasonable solution space to search for a deformation field, and  $R$ and $T$ are two image functions in $C^2(\bar\Omega),$ with  $T,$ and $\nabla T$ being Lipschitz and bounded functions. 
These Lipschitz assumptions are reasonable hypotheses according to the approximation methods described in detail in \cite[pg. 2534]{barnafi2018primal}.

\begin{lemma}
 If $T$ and $\nabla T$ are Lipschitz and bounded functions then
$R-T_u$ and $\nabla T_u$ are also Lipschitz functions with respect to the variable $u$.
\end{lemma}
\begin{proof}
 Let $u_1,u_2\in H$ and $x\in\Omega.$ Since $T$ and $\nabla T$ are assumed to be Lipschitz and bounded functions, it follows that:
\begin{align}
|(T_{u_1}-R)(x) - \,(T_{u_2}-R)(x) |
&=| T(\,x+{u_1}(x)\,)-T(\,x+u_2(x)\,)|\nonumber \\
&\leq Lip( T)| (x+{u_1}(x))-(\,x+u_2(x))| \nonumber\\
&\leq Lip( T) | u_1(x)-{u_2}(x)|,~\forall x\in\Omega~a.e.  \label{eq:LipT}
\end{align}
and
\begin{align}
|\nabla T_{u_1}(x) - \nabla T_{u_2}(x) |
&=|\nabla  T(x+{u_1}(x)) - \nabla T(x+{u_2}(x)) |\nonumber\\
&\leq Lip(\nabla  T)|(x+{u_1}(x)) - (x+{u_2}(x))|\nonumber\\
&\leq Lip(\nabla  T)|{u_1}(x) - {u_2}(x)|, ~\forall x\in\Omega~a.e.  \label{eq:LipGradT}
\end{align}
\end{proof}

Since $R,T\in \mathcal{C}^2(\bar \Omega)$, it can be shown that the Radon transforms of $R$ and $T$ are well defined. Consequently, the similarity measure of $D_{RSSD}$ is well defined by \eqref{eq:DFR}. Using the adjoint of the Radon transform, $D_{R^\#RSSD}$ is also well defined by \eqref{eq:DFRG}.

In the following, the subscripts of $\hat\JC_{RSSD}$, $\hat\JC_{R^\#RSSD}$, $\DC_{RSSD}$, $\DC_{R^\#RSSD}$, $\hat\alpha_{RSSD}$  and $\hat\alpha_{R^\#RSSD}$ may be omitted. In such a case, the context will make it clear which problem is being referred to.

\begin{lemma} The solutions of the Problems \eqref{eq:PropsedDIRProblemAdapted} and \eqref{eq:PropsedDIRProblemAdapted2}
can be written as fixed points of the equation \eqref{eq:WeakDIRProblemNoUnique}, where for each 
$z\in H$, $F_z$ is defined as follow:
\begin{equation}\label{eq:F}
F_z(v):=-\langle f_z,v \rangle_{L^2(\Omega)}, \quad \forall v\in H,
\end{equation}
where  $f_z$  is the nonlinear function:
\begin{equation}\label{eq:fu1}
f_z(x):=\mathcal{R}^\#[\mathcal{R}( T_z-R)](x) \nabla T_z(x),
\end{equation}
for the Problem \eqref{eq:PropsedDIRProblemAdapted}, and
\begin{equation}\label{eq:fu2}
f_z(x):=\mathcal{R}^\#(\mathcal{R} (\mathcal{R}^\#[\mathcal{R}(T_z-R)])) \nabla T_z(x),
\end{equation}
for the Problem \eqref{eq:PropsedDIRProblemAdapted2}.
\end{lemma}

\begin{proof}
First notice that the regularizer $\SC(u)$ can be expressed as: 
\[\SC(u)=\frac{1}{2}a(u,u),\] 
where $a$ is defined in \eqref{eq:regularizator}.

Second, 
\begin{align}
\mathcal{D}_{RSSD}(R,T;u)&
=\dfrac{1}{2}\int_{\Omega}  \mathcal{R}^\#[\mathcal{R}(T_u-R)] (T_u-R)\D{x},\label{eq:RadD} 
\end{align}
whose Gateaux derivative in a direction $v\in H^1(\Omega)$ is given by:
\begin{equation}
D\mathcal{D}(R,T;u)[v]:=\dfrac{d}{d\tau}\mathcal{D}(R,T;u+\tau v)\Big|_{\tau=0}=\int_{\Omega} \RC^\#[\RC( T_u-R )](x) \nabla T_u(x)\cdot v(x) \D{x}. \label{eq:GatDerD}
\end{equation}
The corresponding Euler equation of the Problem \eqref{eq:PropsedDIRProblemAdapted} is
\begin{equation}\label{eq:WeakDIRProblem}
a(u,v)-\hat\alpha F_u(v)=0,\quad \forall v\in H^1(\Omega)
\end{equation}
where
\begin{equation}
F_u(v)=-\int_{\Omega} \RC^\#[\RC( T_u-R )](x) \nabla T_u(x)\cdot v(x) \D{x}
\end{equation}
and this is exactly \eqref{eq:F}, with $u=z$ in \eqref{eq:fu1}.

The analysis for Problem \eqref{eq:PropsedDIRProblemAdapted2} is analogous, noting that, due to the linearity of $\RC$ and $\RC^\#$, the similarity measure \eqref{eq:DFRG}  is equivalent to:
\begin{align}
	\mathcal{D}_{R^\#RSSD}(R,T;u)&
	=\dfrac{1}{2}\int_{\Omega}  \mathcal{R}^\#(\mathcal{R} (\mathcal{R}^\#[\mathcal{R}(T_u-R)])) (T_u-R)\D{x}.\label{eq:RadD2} 
\end{align}
\end{proof}

The following lemma proves the existence and uniqueness of the fixed points of the Problems \eqref{eq:PropsedDIRProblemAdapted} and \eqref{eq:PropsedDIRProblemAdapted2}, for sufficiently large regularization parameters, $\alpha_{RSSD}$ and $\alpha_{R^\#RSSD}$, respectively.
It is shown that the RSSD and $R^\#RSSD$ similarity measures satisfy the necessary hypotheses of Theorem \eqref{th:Barnafi2018}.

\begin{lemma}
The operators $\mathcal T$ related to the Problems \eqref{eq:PropsedDIRProblemAdapted}  and \eqref{eq:PropsedDIRProblemAdapted2}, with $\Omega\subseteq \RR^2$, have at least one fixed point $u\in H,$  if $$\hat\alpha<\dfrac{K}{C_B},$$
where $K$ is  the Korn's related constant, and $C_B$ is the same constant of \eqref{eq:C_B}.

This fixed point $u$  will be unique if $u\in H$ and
$$\hat\alpha<\frac{1}{CL_F}.$$
\label{lm:lemma}
\end{lemma}

\begin{proof}

Since  $T_u-R$ and $\nabla T_u$ are Lipschitz functions, the function  $f_u$ defined in \eqref{eq:fu1} is a Lipschitz function  with respect to $u$, since, given $u_1,u_2\in H,$ it has
\begin{align*}
|f_{u_1}(x)&-f_{u_2}(x)|
= 
\left| \, 
\mathcal{R}^\#[\mathcal{R}( T_{u_1}-R )] (x)\nabla T_{u_1} (x)-
\mathcal{R}^\#[\mathcal{R}( T_{u_2}-R )] (x)\nabla T_{u_2}(x)
\right|\nonumber\\
&= 
\left|
\left(\mathcal{R}^\#[\mathcal{R}( T_{u_1}\,-R )](x) - 
\mathcal{R}^\#[\mathcal{R}( T_{u_2}-R)](x)
\right)\nabla T_{u_1} (x)\right.\\
&\left.\qquad\qquad\qquad\qquad\quad\qquad\qquad+\mathcal{R}^\#[\mathcal{R}( T_{u_2}-R)](x) \left(\nabla T_{u_1}(x)- \nabla T_{u_2}(x)\right)
\right|\nonumber\\
&\leq \left| \, \mathcal{R}^\#[\mathcal{R}( T_{u_2} -T_{u_1})](x)\right|  \left| \nabla T_{u_1}(x)\right| 
\\
&\qquad\qquad\qquad\qquad\quad\qquad\qquad+ \left|\mathcal{R}^\#[\mathcal{R}( T_{u_2}-R)](x)   \right|  \left|  \nabla T_{u_1}(x)  - \nabla T_{u_2}(x) 
\right|\nonumber\\
&\leq \left\|\mathcal{R}^\#\right\|  \left|\mathcal{R}\right\| |T_{u_2}(x)-T_{u_1}(x)|\\
&\qquad\qquad\qquad\qquad\quad+  \left\|\mathcal{R}^\#\right\|  \left\|\mathcal{R}\right\|\sup |T_{u_2}(x)-R(x)|\left|  \nabla T_{u_1}(x)  - \nabla T_{u_2}(x) 
\right|\\
&\leq [ \,K_1 Lip(T) +K_2 Lip(\nabla T)\, ] \, |u_1(x) -u_2(x) |,\\
&= L_F|u_1(x) -u_2(x) |, ~\forall x\in\Omega~a.e. ,\nonumber
\end{align*}
where $K_1,K_2$ are positive constants, such that 
$ \|\mathcal{R}^\#\| \, \|\mathcal{R} \| \,\| \nabla T_{u_1} \|\leq K_1$,
\linebreak
 $ \|
\mathcal{R}^\#\| \, \|\mathcal{R} \| \,\sup|T_{u_2}(x)-R(x)| 
\leq K_2,$ and $L_F =  K_1 Lip(T) +K_2 Lip(\nabla T)$, where $Lip(T)$ and $Lip(\nabla T)$ are constants given in \eqref{eq:LipT} and \eqref{eq:LipGradT}.
It follows that,
\begin{equation} 
	\|f_{u_1}-f_{u_2}\|_{L^2{(\Omega)}} \leq L_F\|u_1 -u_2 \|_{L^2{(\Omega)}} .
	\label{eq:fuL2}
\end{equation}

Using \eqref{eq:fuL2} and Hölder's inequality, it holds that:
 \begin{align*} 
 	\|F_{u_1}-F_{u_2}\|_{H^\prime} 
 	& =\sup_{ v\in H, ~\|v\|_H\leq 1} \left|\langle f_{u_1}- f_{u_2},v \rangle_{L^2(\Omega)}\right| \label{eq:FFLip}
 	\\  &\nonumber
 	\leq \sup_{ v\in H, ~\|v\|_H\leq 1} \| f_{u_1}- f_{u_2}\|_{L^2(\Omega)} \|v \|_{L^2(\Omega)} \\ \nonumber&
 	 	{\leq \| f_{u_1}- f_{u_2}\|_{L^2(\Omega)} \sup_{ v\in H, ~\|v\|_H\leq 1}  \|v \|_{H} }\\ &\nonumber
 	 	\leq  L_F\|u_1 -u_2 \|_{L^2{(\Omega})},
 \end{align*}
verifying the assumption \textbf{H1} of Theorem \eqref{th:Barnafi2018}.

From the linearity and continuity of $\mathcal{R}$ and $\mathcal{R}^\#$, it follows that $\mathcal{R}^\#(\mathcal{R}(T_u-R))$ belongs to $L^2(\Omega)$. Hence, since to $\nabla T_u$ and $T_u$ are assumed to be bounded (recall that $u$ can be replaced by $u_M$ bounded), one finds that $f_u\in L^2(\Omega)$. Then,
\begin{equation}
\left|F_u(v) \right|=\left|\langle f_{u},v \rangle_{L^2(\Omega)} \right|\leq \|f_u\|_{L^2(\Omega)} \|v\|_{L^2(\Omega)}\leq C_u \|v\|_{L^2(\Omega)}\leq C_u \|v\|_{H^1(\Omega)},\label{eq:FuLC}
\end{equation}
where $C_u$ is a continuity constant. It follows that $F_u\in (H^1(\Omega))^\prime\subseteq H^\prime,$ verifying the assumption \textbf{H2} of Theorem \eqref{th:Barnafi2018}.

Therefore, thanks to the Theorem \ref{th:Barnafi2018}, one obtains that $\mathcal{T}$ has at least one fixed point, and these fixed points are unique if $\hat\alpha C L_F<1$.   
\end{proof}

Since this existence of fixed points implies the existence of a solution for the weak Problem \eqref{eq:WeakDIRProblem}, in the following theorems conditions for the existence and uniqueness of solutions for the Problems \eqref{eq:PropsedDIRProblemAdapted}  and \eqref{eq:PropsedDIRProblemAdapted2} are established.

\begin{theorem}
	The Problems \eqref{eq:PropsedDIRProblemAdapted}  and \eqref{eq:PropsedDIRProblemAdapted2} with $\Omega\subseteq \RR^2$,  
	 have at least one solution $u\in H,$  if 
		$$\hat\alpha<\dfrac{K}{C_B},$$
where $K$ is  the Korn's related constant, and $C_B$ is the same constant of \eqref{eq:C_B}.

This solution $u$  will be unique if $u\in H$ and
	$$\hat\alpha<\min\left\{\dfrac{ K}{C_B},\frac{1}{CL_F}\right\}.$$
\end{theorem}

\begin{proof}

The existence of solutions for the weak Problem \eqref{eq:PropsedDIRProblemAdapted}   is deduced due to the existence of fixed points of the operator $\mathcal T$, given by Lemma \ref{lm:lemma}. Furthermore, it is now necessary to verify that the solutions found for the Problem \eqref{eq:WeakDIRProblem} are in fact minimizers for the proposed DIR Problem \eqref{eq:PropsedDIRProblemAdapted}. 
In order to demonstrate this, it will be shown that $$\hess \hat\JCR (u^*)[v,v]> 0, \qquad \forall v\in H\setminus \{0\},$$ where $u^*$ is the fixed point of the operator $\mathcal T$ found earlier, i.e., it is a solution of the weak Problem \eqref{eq:WeakDIRProblem}, with $z=u^*$.

It is known that:
\[\hess \hat\JCR (u)[v,v]=\hess \SC (u)[v,v]+\hat\alpha\hess \DC(R,T;u)[v,v].\]
On the one hand, note that the derivatives of $a$ are given by:
\[
\jac a(u,u)[v]=2\int_\Omega \mathbb{C}\varepsilon(u):\varepsilon(v)\D x ~~\mbox{and}~~
\hess a(u,u)[v,w]=2\int_\Omega \mathbb{C}\varepsilon(w):\varepsilon(v)\D x.
\]
Additionally, from the Poincaré inequality and from 
$\trace(\varepsilon^2(v))=\frac{1}{2}|\nabla v|^2,$
 $D^2a$ is coercive because:
\begin{align*}
\hess a(u,u)[v,v]&=2\int_\Omega \mathbb{C}\varepsilon(v):\varepsilon(v)\D x
= 2 a(v,v) \geq 2K\|v\|^2_{H}, \qquad 
\end{align*}
where $K$ is a constant from the Korn's inequality 3.4, in \cite{duvant2012inequalities}.

On the other hand, the first Gateaux derivative of $D$ is given in \eqref{eq:GatDerD}, and its second derivative is given by:
\begin{align*}
D^2\mathcal{D}(R,T;u)[v,&w]:=\dfrac{d}{d\mu}\dfrac{d}{d\tau}\mathcal{D}[R,T;u+\tau v+\mu w]\Big|_{\tau=\mu= 0}\nonumber\\
&=\!
\int_{\Omega} \;\mathcal{R}^\#[\mathcal{R}(\nabla T_u\cdot v )] (\nabla T_u\cdot  w)\D x
- \!\int_{\Omega} \!\mathcal{R}^\#[\mathcal{R}( R-T_u )]  D^2T_u[v,w].
\end{align*}
Note that, the first term of $D^2\DC(u)[v,v]$ is non-negative, because, for all $ v\in\VC$:
\[\int_{\Omega} \mathcal{R}^\#[\mathcal{R}(\nabla T_u\cdot v )](x) (\nabla T_u\cdot  v)(x)\D x = \int_{S^1\times\RR} [\mathcal{R}(\nabla T_u\cdot v )(\sigma,s)]^2 \D\sigma \D s\geq 0.
\]
Since the Radon transform is continuous, $R,T\in C_0^2(\Omega),$ and $T$ Lipschitz, the second term of $\hess\hat\JCR(u)[v,v]$ is bounded.  
Let $u^*$ be the fixed point of the operator $\mathcal T$  previously found. For all    $v\in H\setminus \{0\}$, this follows:
\begin{align}
\biggl|\int_{\Omega} \mathcal{R}^\#[\mathcal{R}( R-T_{u}  )]  D^2T_u[v,v]\D x\biggl|
&\leq \int_{\Omega} \left|\mathcal{R}^\#[\mathcal{R}( R-T_{u} )]  
D^2T_u[v,v]\right|\D x\nonumber\\
&\leq  \left\| \mathcal{R}^\#\mathcal{R}( R-T_{u}) \right\|_{L^2(\Omega)}   
\left\| D^2T_u[v,v]\right\|_{L^2(\Omega)}\nonumber\\
&\leq  \left\| \mathcal{R}^\#\mathcal{R}( R-T_u) \right\|_{L^2(\Omega)}   
\left\| D^2T_u[v,v]\right\|_{L^2(\Omega)}\nonumber\\
&\leq  \left\|\mathcal{R}^\#\right\|  \left\|\mathcal{R}\right\|  
\left\|  R-T_{u}\right\|_{L^2(\Omega)}   
\left\| D^2T_u[v,v]\right\|_{L^2(\Omega)}\nonumber\\
&  \leq K_B \, \left\|\mathcal{R}^\#\right\|  \left\|\mathcal{R}\right\|   
\|v\|_{L^2(\Omega)}^2\nonumber\\
&\leq   C_B\|v\|_{H}^2 ,  \label{eq:C_B}
\end{align}
where $K_B$ is a  boundedness  constant, since $R-T_{u}$ and $D^2T_u$ are bounded (recall again that $u$ can be assumed to be bounded) and $C_B= K_B\left\|\mathcal{R}^\#\right\|  \left\|\mathcal{R}\right\|   $.

Since $D^2a$ is coercive, the first term of $D^2\DC$ is positive, and the second term of $D^2\DC$ is bounded, it follows that:
\begin{align}
\hess \hat\JCR (u^*)[v,v]&=\hess \SC (u^*)[v,v]+\hat\alpha\hess \DC(R,T;u^*)[v,v]\nonumber\\
&=\dfrac{1}{2}\hess  a({u^*},{u^*})[v,v]+\hat\alpha\hess \DC(R,T;u^*)[v,v]\nonumber\\
&\geq K\|v\|^2_{H}+\hat\alpha
\int_{\Omega} \mathcal{R}^\#[\mathcal{R}(\nabla T_{u^*}\cdot v )] (\nabla T_{u^*}\cdot  v)\D x\nonumber\\
&\qquad \qquad
-\hat\alpha\int_{\Omega} \mathcal{R}^\#[\mathcal{R}( {T_{u^*}-R })]  D^2T_{u^*}[v,v]\D x
\nonumber\\
&\geq  K\|v\|^2_{H}+\underbrace{\hat\alpha
\int_{\Omega} [\mathcal{R}(\,\nabla T_{u^*}\cdot v \,)] ^2\D x }_{\geq 0} -\hat\alpha C_B\|v\|^2_{H}
\nonumber\\
&\geq   (K-\hat\alpha C_B)\|v\|^2_{H}. \nonumber
\end{align}
And, by imposing that $\hat\alpha<\dfrac{ K}{C_B},$ 
\[
\hess \hat\JCR (u^*)[v,v]>  0.
\]
From the continuity of the functional $\hat\JCR$ and the fact that 
$\hess \hat\JCR (u^*)[v,v]>0,$ one finds that $u^*$ is also a local minimum of $\hat\JCR$.
And, in accordance with the Lemma \ref{lm:lemma}, this solution $u^*$  will be unique if 
	$$\hat\alpha<\min\left\{\dfrac{ K}{C_B},\frac{1}{CL_F}\right\}.$$

The proof of the previous result for the Problem \eqref{eq:PropsedDIRProblemAdapted2} is also analogous.
\end{proof}

In addition, to prove that $u^*$ is the minimum of the functional $\hat\JCR$, an alternative demonstration using a Taylor expansion is included.
The idea is to show that $\hat\JCR$ is a convex functional around $u^*$, i.e., for all directions $v \in H$ it holds that
$\hat\JCR[u^*+v]>\hat\JCR[u^*].$

\begin{theorem}
	The Problems \eqref{eq:PropsedDIRProblemAdapted}  and \eqref{eq:PropsedDIRProblemAdapted2} with $\Omega\subseteq \RR^2$, have at least a solution $u\in H,$  if $$\hat\alpha<\frac{C_{coe}}{2C_{BC}},$$
	where $C_{coe}$ is a coercitivity constant and  $C_{BC}$ is a continuity constant.
	
	This solution $u$ will be unique if  $u\in H$ and
	$$\hat\alpha<\min\left\{\frac{C_{coe}}{2C_{BC}},\frac{1}{CL_F}\right\}.$$
\end{theorem}

\begin{proof}

Due to the existence of fixed points of the operator $\mathcal T$, given by Lemma \ref{lm:lemma}, the existence of solutions for the weak Problem \eqref{eq:PropsedDIRProblemAdapted} is deduced. Furthermore, it must be verified that the fixed points found for \eqref{eq:WeakDIRProblem} are indeed minimizers for the proposed DIR Problem \eqref{eq:PropsedDIRProblemAdapted}. 

Let $v\in H\setminus\{0\}$ be an arbitrary direction. Given the $u^*$ solution to Problem \eqref{eq:WeakDIRProblem}, it follows that $a(u^*,v)-\hat\alpha F_{u^*}(v)=0.$  Furthermore, from the symmetry and positively of the bilinear form $a,$ and a Taylor expansion of $\DC$ at $u^*$ in the direction $v$, one obtains that:
\begin{align}
	\hat\JCR[u^*+v]&:=\SC[u^*+v]+\hat\alpha \DC[R,T;u^*+v]\nonumber\\ 
	&=\dfrac12 a(u^*,u^*) + a(u^*,v)+\dfrac12 a(v,v) +\hat\alpha \{\DC[R,T;u^*] 
	\nonumber\\
	&\qquad\qquad+ \nabla \DC[R,T;u^*]\cdot v + r_{u^*,2}(v)
	\} \nonumber\\
	&=\left\{\dfrac12 a(u^*,u^*)+\hat\alpha D[R,T;u^*]\right\}  +\big\{ a(u^*,v)
	-(-\hat\alpha \nabla \DC[R,T;u^*]\cdot v )\big\}
	\nonumber\\
	&\qquad\qquad+\dfrac12 a(v,v)  + \hat\alpha  r_{u^*,2}(v)\nonumber\\
	&=\hat\JCR[u^*] +\{ \underbrace{ a(u^*,v)-\hat\alpha F_{u^*}(v) }_{=0} \}+\dfrac12 a(v,v)  + \hat\alpha  r_{u^*,2}(v) \nonumber\\
	&=\hat\JCR[u^*]+\dfrac12 a(v,v)+ \hat\alpha  r_{u^*,2}(v), \nonumber
\end{align}
where 
$ r_{u^*,2}(v):=\dfrac{1}{2}\hess\DC[R,T;u^*+cv][v,v]$ is the remainder, with $c\in(0,1).$
Let $\hat\alpha>0$ and $a(v,v)>0.$ If $ r_{u^*,2}(v)\geq 0,$ then 
$\hat\JCR[u^*+v]> \hat\JCR[u^*].$ If $ r_{u^*,2}(v)< 0,$
 it is necessary that:
\begin{equation}
\hat\alpha <	\dfrac{a(v,v)}{|2r_{u^*,2}(v)|} ,\quad \forall v\in H. 
\end{equation}
Since $a$ is coercive, there exists $C_{coe}$ such that $a(v,v)\geq C_{coe}\|v\|^2.$ And since the  expression $\hess\DC[R,T;u^*+cv][\cdot,\cdot]$ is a continuous bilinear form,  there exist $C_{BC}>0,$ such that: $$ |r_{u^*,2}|\leq\Big|\dfrac{1}{2}\hess\DC[R,T;u^*+cv][v,v]\Big|\leq C_{BC}\|v\|^2.$$
So if we take $\hat\alpha<\frac{C_{coe}}{2C_{BC}},$ it follows:
\begin{equation}
	\left|\dfrac{a(v,v)}{2r_{u^*,2}(v)}\right|\geq \dfrac{C_{coe}\|v\|^2}{C_{BC}\|v\|^2}\geq \dfrac{C_{coe}}{C_{BC}} >\hat\alpha.
	\label{eq:condalpha}
\end{equation}
It can thus be concluded that if the condition  \eqref{eq:condalpha} is satisfied, the fixed point $u^*$ is also a local minimum of $\hat\JCR$. And, in accordance with the Lemma \ref{lm:lemma}, this solution $u^*$ will be unique if 
$$\hat\alpha<\min\left\{\frac{C_{coe}}{2C_{BC}},\frac{1}{CL_F}\right\}.$$

The proof of the previous result for the Problem \eqref{eq:PropsedDIRProblemAdapted2} is also analogous.

\end{proof}

It should be noted that the results previously obtained for the reformulated Problems \eqref{eq:PropsedDIRProblemAdapted}  and \eqref{eq:PropsedDIRProblemAdapted2} are also applicable to the Problems \eqref{eq:ProposedDIRProblem} and \eqref{eq:ProposedDIRProblem2}.

\section{Numerical implementation}
\label{sec:Implementation}

In the following section, a finite element-based method is proposed to compute the numerical solution of the DIR Problems \eqref{eq:ProposedDIRProblem} and \eqref{eq:ProposedDIRProblem2}. 
Let  $\Omega^h\subset\Omega$ be the image domain composed of the elements $\Omega^e$  satisfying that $\Omega^h= \cup^{n_{el}} _{e=1} \Omega^e.$ 
The finite dimensional functional space $\VC^h\subset\VC$ is defined by: 
\[
\VC^h :=\Big\{ u^h\in\VC\,|\, u^h(x)= \dsum_{A=1}^{N_{dof}} N_A(x) u_A\Big\},
\]
where $N_{dof}$ is the number of degrees of freedom, $\{ N_A(x) : A=1,\cdots,N_{dof}\}$ is a set of the basis functions of the space $\VC^h$, also called the shape functions, and $u_A$ are the nodal displacements of the vertices $x_A$ of the elements in the domain $\Omega^h$. 
The basis functions are considered as continuous piecewise linear polynomials over the domain $\Omega^h$, satisfying that the Kronecker delta identity in the nodes, \textit{i.e.}, $N_A(x^B)=\delta_{AB}$.
In this way, the deformation field $u\in\VC$ can be approximated by an element $u^h\in\mathcal{V}^h$ which satisfies: \[u(x) \approx u^h(x):=\dsum_{A=1}^{N_{dof}} N_A(x) u_A,\]
where all vertices have free motion, including the nodes at the edges of the image.

For the calculation of the deformations, two triangular Delaunay meshes were considered, as shown in Figure \ref{fig:meshes}. A \textit{coarse} mesh with 64 elements, 41 nodes, and  82 degrees of freedom, and a \textit{fine} triangular mesh with 545 elements, 1024 nodes and 2048 degrees of freedom. 

\begin{figure}[t!]
	\centering
	\begin{minipage}[t]{\linewidth}
		\centering
		\subfigure[Coarse mesh:  $N_{el}=64$ and $N_{dof}=82$.]{\includegraphics[width=0.45\linewidth]{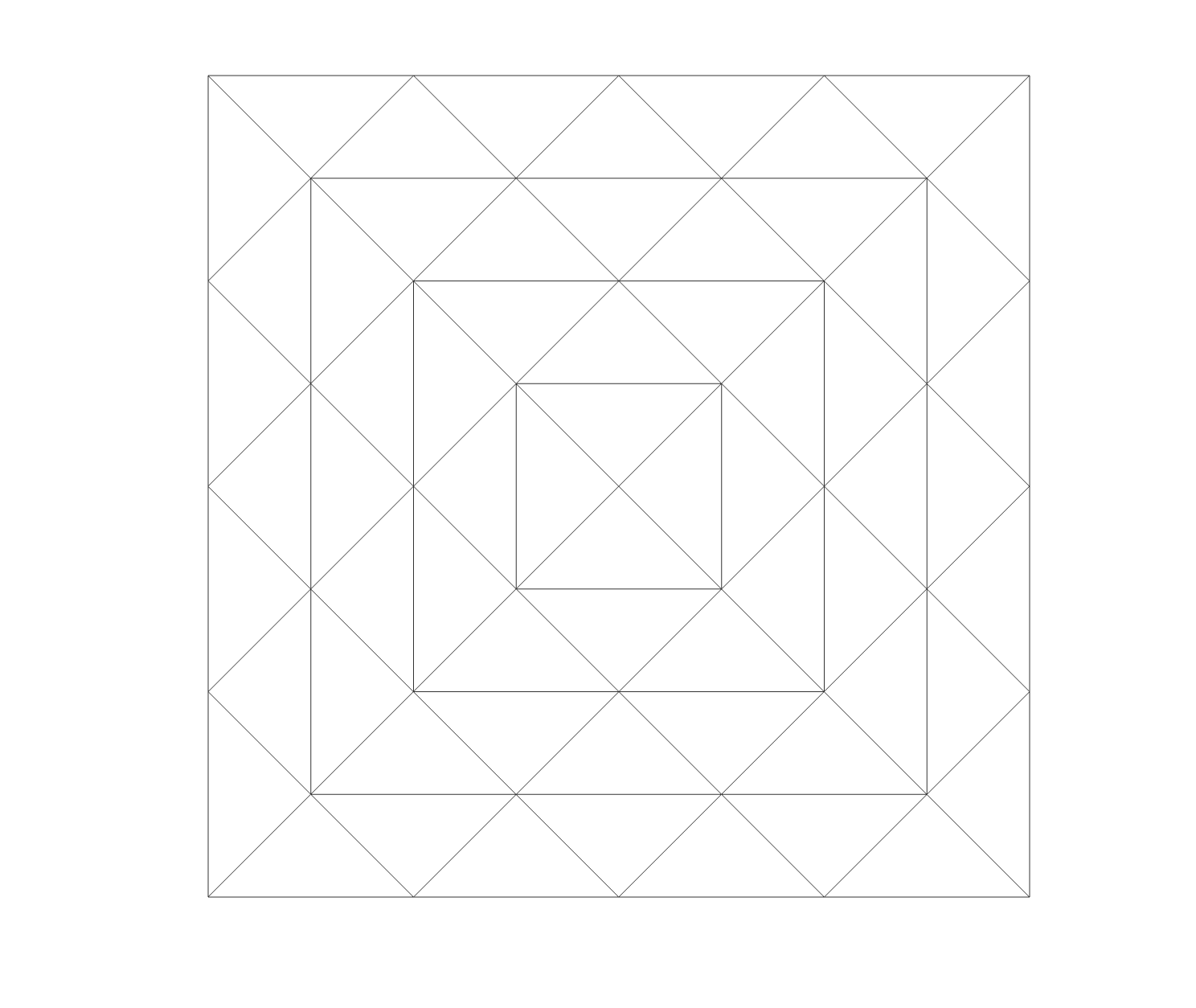} \label{fig:mesh} }\quad
		\subfigure[Fine mesh:   $N_{el}=545$ and $N_{dof}=2048$.]{\includegraphics[width=0.45\linewidth]{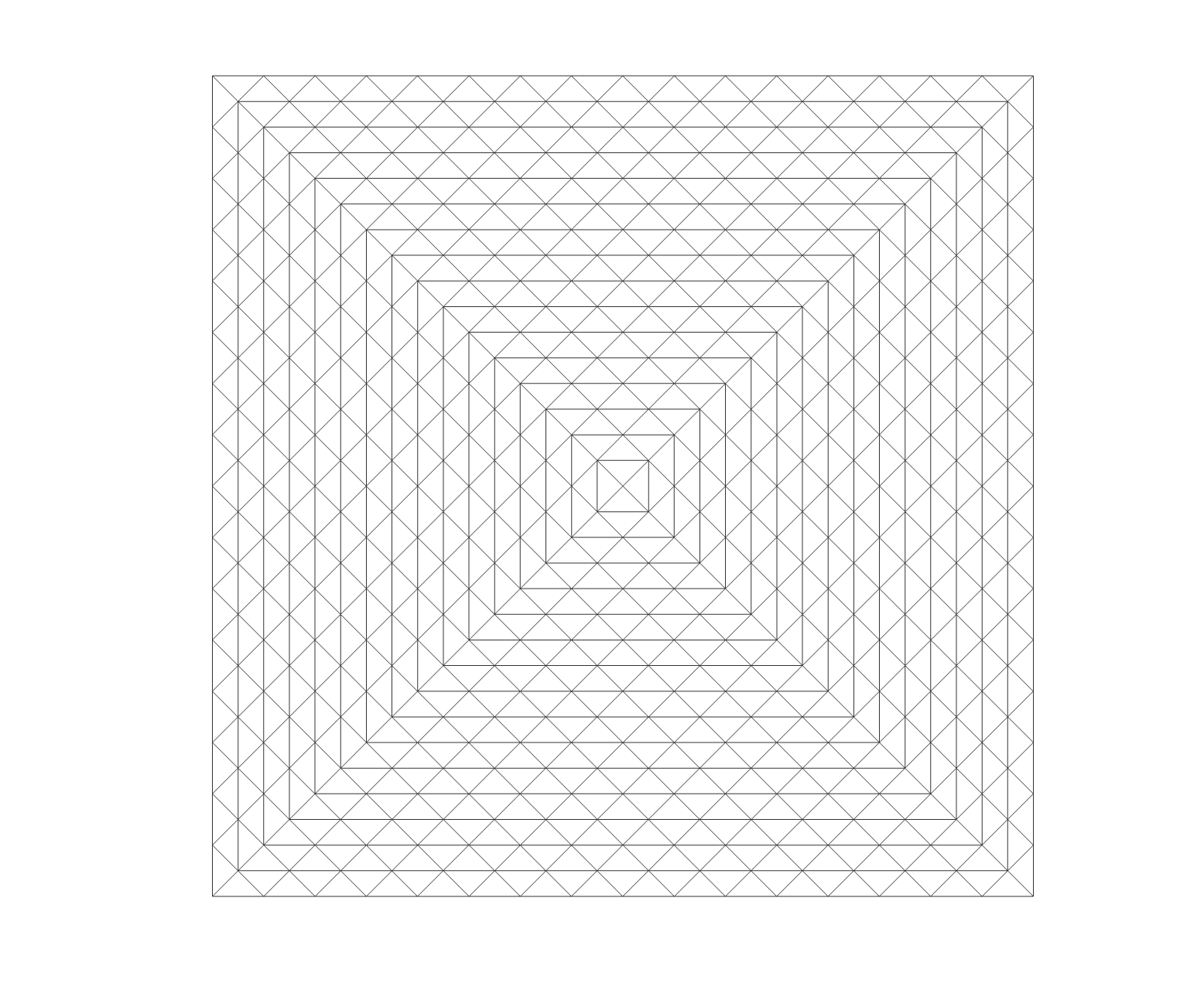} \label{fig:meshref} }
	\end{minipage}
	\caption{Delaunay meshes utilized in this work.  }
	\label{fig:meshes}
\end{figure}

The discretization of the LEE regularizer \eqref{eq:LE} leads to the quadratic form
\begin{equation}
	\SC^h(u^h):= \mathbf{u}^t\,\mathbf{K}\,\mathbf{u},
	\label{eq:sch}
\end{equation}
where  $\textbf{u}=(u_{A_1},\cdots,u_{A_{N_{dof}}}),$ and $\textbf{K}$ is the stiffness matrix, computed by 
$K:=\int_{\Omega^h} B^T_ADB_A \,\mbox{d}\Omega^h,$ where $D$ is the Cauchy tensor, $B$ is an array representing the term $\varepsilon(N_A(x))$ in equation \eqref{eq:LE}, $N_{A,1}$ and $N_{A,2}$ are the derivatives of the FEM basis function, and  \[B^T_A:=\begin{bmatrix}
	N_{A,1}&0&N_{A,2}\\
	0&N_{A,2}&N_{A,1}
\end{bmatrix}.\] 
For details, see \cite[Sections 2.7-2.9]{hughes2012finite}.

The similarity measure $\DC_{SSD}$ is discretized as follows:
 $$\DC^h_{SSD}(R,T;u):=\frac{4}{N^2}\, \|R-T_{u^h}\|_2^2,$$  
where $u^h\in\VC^h$ is the deformation field,  $R$ and $T_{u^h}$ are images of size $N\times N$,  $T_{u^h}$ is the interpolated image which takes the intensities of $T$ in the pixel $(i,j)$, and assigns them to the pixel  $(i,j)+u^h$.

The similarity measure $\DC_{RSSD}$ is discretized as:
\begin{align}
\DC^h_{RSSD}[R,T;u^h] :=\dfrac{1}{2}\cdot\dfrac{2\pi}{N_\omega\,N_s}\sum_{\omega=1}^{N_{\omega}}\sum_{s=1}^{N_{s}} &\big(\,\RC\big(T_{u^h}\big)[\omega,s]-\RC(R)[\omega,s ]\,\big)^2
\label{eq:RSSD_Discrete}
\end{align}
where  $N\times N$ pixels, $N_\omega$ is the number of equidistant nodes in $[0,\pi],$
$N_s$ is the number of equidistant nodes in $[-1,1],$  and where $u^h,R$ and $T_{u^h}$  are the same as before.

From the linearity of the Radon transform,  $D^h_{RSSD}$ is calculated directly from the sinograms of the images $R$ and $T_u^h$, as follows:
\begin{equation}
\label{eq:RSSDFromSinogram}
\DC^h_{RSSD}[R,T;u^h] =\dfrac{1}{2}\cdot \dfrac{2\pi}{N_\omega\,N_s}\, \|\mbox{sinogram}( T_{u^h} )-\mbox{sinogram}(R)\|^2,
\end{equation}
where \text{sinogram}(R) and \text{sinogram}(${u^h}$) are real matrix of size $N_s\times N_\omega $.

The similarity measure $\DC_{R^\#RSSD}$ is discretized as:
\begin{align}
	\DC^h_{R^\#RSSD}[R,T;u^h] :=\dfrac{1}{2}\cdot\dfrac{4}{N^2}\sum_{i,j=1}^{N} &\big(\,\RC^\#(\RC\big(T_{u^h}\big))(i,j)-\RC^\#(\RC(R))(i,j)\,\big)^2
	\label{eq:RSSD_Discrete2}
\end{align}
where $u^h,R,T, T_{u^h}$  and $N$ are the same as before.

The term $D^h_{R^\#RSSD}$ can be calculated directly from the sinograms of $R$ and $T_u^h$, as follows:
\begin{equation}
	\label{eq:RSSDFromSinogram2}
	\DC^h_{R^\#RSSD}[R,T;u^h] =\dfrac{1}{2}\cdot \dfrac{4}{N^2}\, \|\mbox{BP}(\mbox{sinogram}(T_{u^h}))-\mbox{BP}(\mbox{sinogram}(R))\|^2,
\end{equation}
where BP(\text{sinogram}(R)) and BP(\text{sinogram}(R)) represent real matrices of size $N\times N $,  corresponding to the back-projection of the sinograms. Accordingly, the functionals $\JC_{RSSD}[u]$ and $\JC_{R^\#RSSD}[u]$ can be discretized. Since the measures $\DC^h_{SSD}$, $\DC^h_{RSSD}$ and $\DC^h_{R^\#RSSD}$ have different ranges and are dissimilar, scaling adjustment factors were been integrated into the similarity measures. This enables the regularization parameters of the three methods to be set within a narrow range. It is noteworthy that these modifications do not affect the effectiveness of the methods.

In terms of images and sinograms, the following experiments use grayscale images of size $128\times128$ pixels are utilized, with intensities in the range $[0,1],$ and with image domain $\Omega=[-1,1]^2$.  One of the main challenges in this work is to perform registrations using noisy images, such as medical images. For this reason, this work considers images with two levels of noise: a \textit{low noise} of type white Gaussian with mean 0 and variance $0.05^2$ and a \textit{high noise} of type white Gaussian with mean 0 and variance $0.1^2$.

To discretize the Radon transform, the variable $s$ was considered to be in the interval $[-1,1]$, with $185$ equidistant distances, while the variable  $\omega$ was considered to be in the interval $[0,\pi]$, with $360$ equidistant angles.  Consequently, the values of $h_s$ and $h_\omega$ are given by  $2/128$ and $\pi/180$, respectively.

The numerical computations are done using the \texttt{Matlab R2022B} software. The \texttt{radon} script function is used to construct the sinograms. Subsequently, the back projection of the sinograms is constructed using the \texttt{iradon} script function with the configuration \texttt{filter=none}

To minimize the DIR problem, two existing Matlab programs are employed: The optimization algorithms used are  \texttt{fminunc} and \texttt{fmincon}. By default, the first uses an algorithm  \texttt{quasi-newton} type, and the second uses an algorithm of type \ \texttt{ interior-point} type. Both methods use algorithms of type \texttt{bfgs} for the Hessian update.
Furthermore, to improve the methodology, the analytical gradient of both functionals has been incorporated. The discretization of these gradients was performed using the same techniques  described previously.

Finally, two standard measures are presented to compare the similarity of images. The root mean square error (RMSE) and the norm of the difference of the deformation fields computed over the object of interest.
The RMSE is a statistical measure used to quantify the discrepancy between two variables. In the case of square images of equal size,  $N\times N$, it is defined by:
$$RMSE(R,T) = \sqrt{\dfrac{1}{N^2} \dsum_{i,j=1}^N(R_{ij}-T_{ij})^2},$$
where $R_{ij}$ and $T_{ij}$ are the pixels at  position $(i,j).$ 
A RMSE value equal to zero indicates that the images are identical. 

Let $F=(F_1,F_2)$ be a referential deformation field and $G=(G_1,G_2)$ be the registration deformation field. The norm of the difference of the deformation fields computed over the object of interest is defined as follows:
$$\|F-G \|= \sqrt{\|F_{1}-G_1\|^2_2 + \|F_2-G_2\|^2_2},$$
where $F_i$ and $G_i$ are the $i$-th component of vectorial fields $F$ and $G$, respectively.

\section{Numerical Tests}
\label{sec:ExperimentalResults}

The following section presents five tests that examine the performance of the three methods discussed in this paper. The registration of different types of non-affine deformations is demonstrated in both noisy and non-noisy images. The selected tests are:

\begin{description}
\item[Test A:] Registration of non-affine deformations in noise-free images.
\item[Test B:] Registration of non-affine deformations in images with random noise.
\item[Test C:] Convergence rates in the non-affine registration of noise-free images.
\item[Test D:] Convergence rates in the non-affine registration of noisy images.
\item[Test E:] Application to the lung deformation.
\end{description}

\subsection{Test A: Registration of non-affine deformations in noise-free images}

\begin{figure}[t!]
	\centering
	\includegraphics[width=0.99\linewidth]{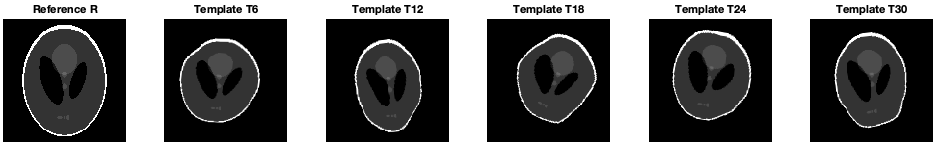}
	\caption{{\textbf{Sampling of images considered in Test A:} Reference R and a sample of five of the random Template images  T}}
	\label{fig:imnonoise}
\end{figure}

\begin{figure}[t!]
	\centering
	\includegraphics[width=1\linewidth]{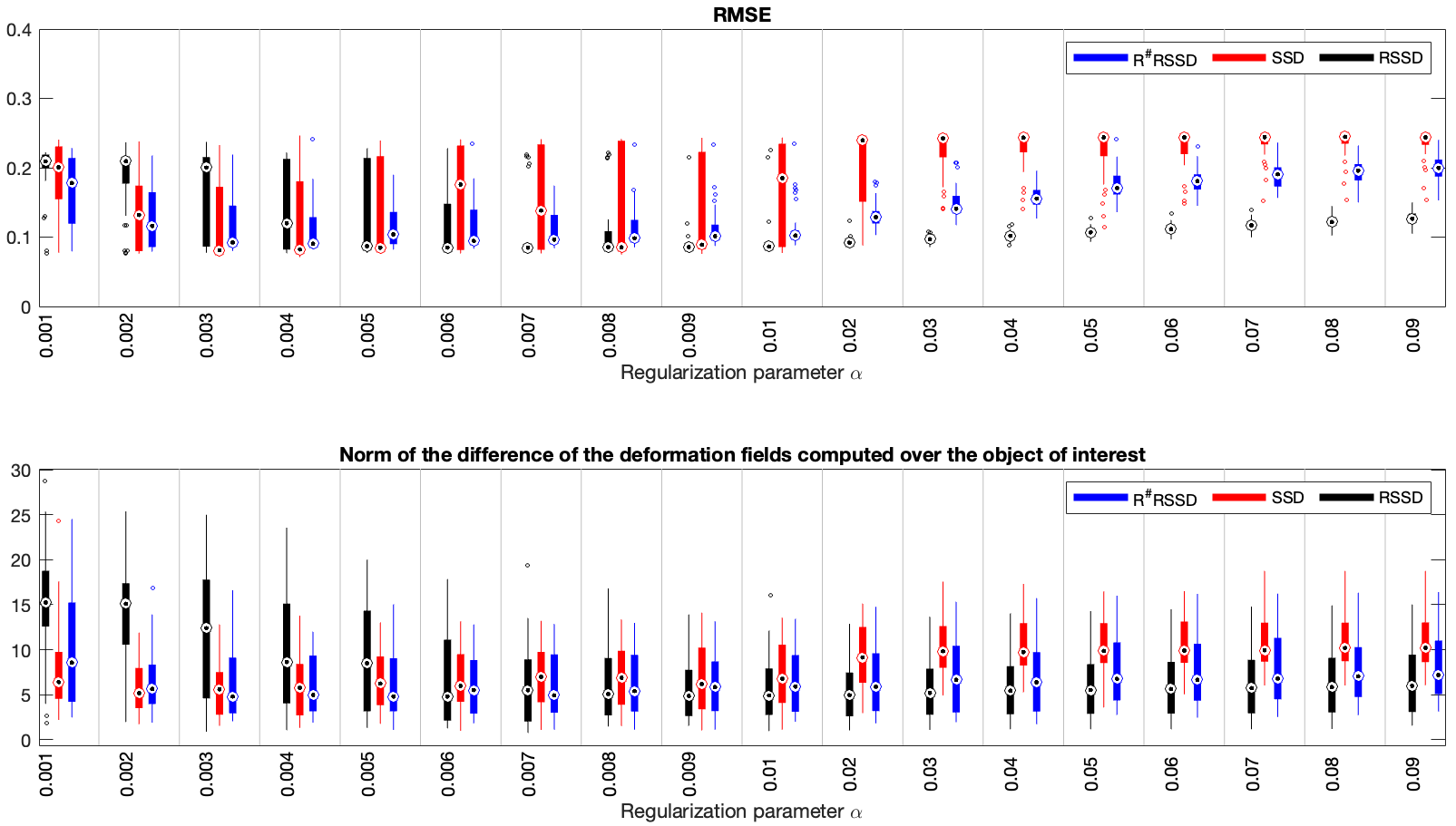}
	\caption[]{\textbf{
 Results of the registration of thirty random deformations  in noise-free images.} 
	The best regularization parameters for each method were as follows: RSSD: $\alpha= 0.02,$  SSD: $\alpha= 0.003,$ y  R$^\#$SSD: $\alpha= 0.007$.
	}
	\label{fig:alphadet}
\end{figure}

In this experiment, the efficacy of the three methods was evaluated using a coarse mesh and thirty non-affine deformation fields on images devoid of noise. 
Shep-Logan phanom of size $128 \times 128$ was considered as reference image $R$.

A total of thirty random deformations of the reference image, designated as template image T, were generated with the same dimensions as R. The methodology employed for this purpose is outlined below. Initially, a global affine deformation with inhomogeneous scaling in the principal directions was applied, with aleatory values within the interval $[0.69,0.91]$. This was achieved through a random rotation within the interval $[-30^\circ,30^\circ]$ and an arbitrary translation of up to nine pixels in all directions.

In order to apply further random deformations subsequent to the aforementioned affine deformation, an additional triangular mesh with 40 nodes was employed. This process resulted in the generation of local deformations. Figure \ref{fig:imnonoise} illustrates an example of these deformations.

Figure \ref{fig:alphadet} illustrate the results of registering these thirty noise-free images. The top graph depicts the box-and-whisker plots of the RMSE measurements obtained by comparing the reference image R with the registered images from the 30 aforementioned experiments.  The RSSD method yields registrations with lower and more homogeneous RMSE values, with a low number of outliers. 
In the case of the SSD method, it can be observed that the RMSE values are generally heterogeneous. However,  for certain $\alpha$ values in the interval $[0.003, 0.005]$, the median values are acceptable. In comparison to the SSD method, the R$^\#$SSD method demonstrates superior performance, although not at the level of the RSSD method.

Given that RSME is a global error metric that does not provide additional information regarding potential local discrepancies, it was determined that incorporating a second metric would be advantageous for more accurately assessing the efficacy of the methods. To this end, the aforementioned norm of the difference of the deformation fields computed over the object of interest was considered. This second metric compares the estimated fields with the real deformation field, thus providing a more precise evaluation of the performance of the methods.
It can thus be verified that the deformation fields obtained through the RSSD method are in agreement with those derived from RMSE for alpha values within the range  $[0.007,0.03]$. The best results were achieved for  $\alpha_{RSSD} = 0.02,$  exhibiting a low median and high homogeneity.

For SSD, it can be observed that the deformation fields effectively align with the RMSE results within the interval  $[0.003, 0.005]$. The best results were obtained for  $\alpha_{R^\#RSSD} = 0.007$, exhibiting a low median and homogeneity of values for both measurements. 
It can be observed that at least half of the registrations yield satisfactory results. However, it is also evident that a considerable number of errors are present.

In the case of R$^\#$SSD, it can be observed that the deformation fields align with the RMSE results within the interval  $[0.003,0.01].$
 It can be noted that it yields a greater number of acceptable registrations than SSD, achieving the best results for $\alpha_{ SSD} = 0.003,$ with a low median and homogeneity of values for both measures.

\begin{figure}[t!]
	\centering
	\includegraphics[width=0.99\linewidth]{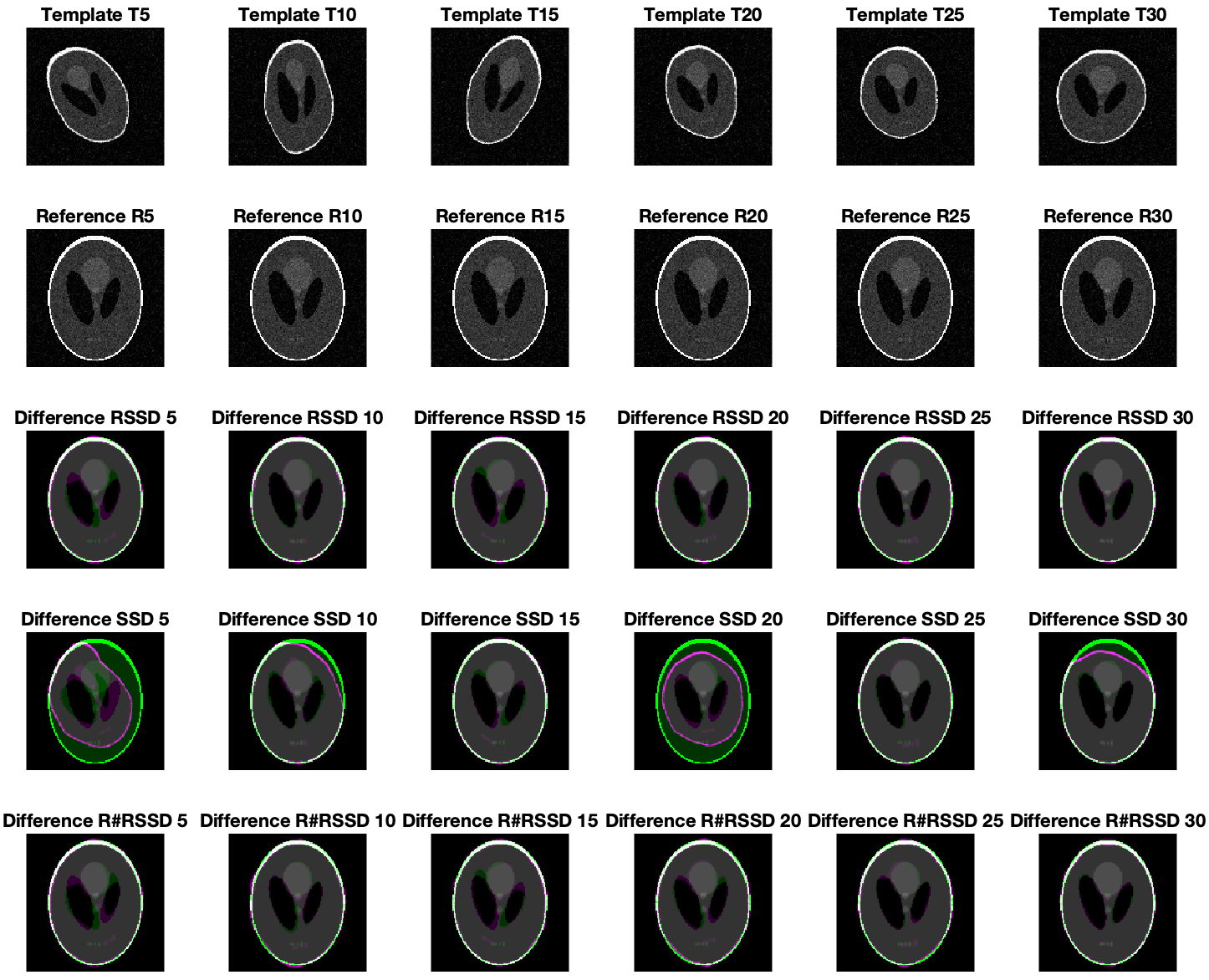}
	\caption{{\textbf{Sampling of images contained in Test B} - Low Noise:  The first row displays the template T images from a sampling of thirty high-noise images. The second row displays their respective six reference images with different seed noise. 
	The third, fourth, and fifth rows displays the difference between the reference image R and the registration image from the RSSD, SSD, and R$^\#$SSD methods, respectively, with the noise removed for comprehension.
	}}
	\label{fig:SampleNoise5}
\end{figure}

With regard to the deformation fields, there is a clear alignment between the RMSE results and the deformation fields in the interval $[0.003, 0.005].$ The best results are observed for the value of  $\alpha_{R^\#RSSD} = 0.007$, which demonstrates low medians and homogeneity for both measurements. While at least half of the registrations yielded positive results, there was a notable prevalence of failures.

In the case of R$^\#$SSD, the deformation fields align the RMSE results in the interval $[0.003,0.01].$ Consequently, it can be stated that R$^\#$SSD yields a more greater number of acceptable registrations in comparison to SSD. The optimal results are observed for $\alpha_{SSD} = 0.003,$ which demonstrates low medians and homogeneity for both measures.

\begin{figure}[t!]
	\centering
	\includegraphics[width=0.99\linewidth]{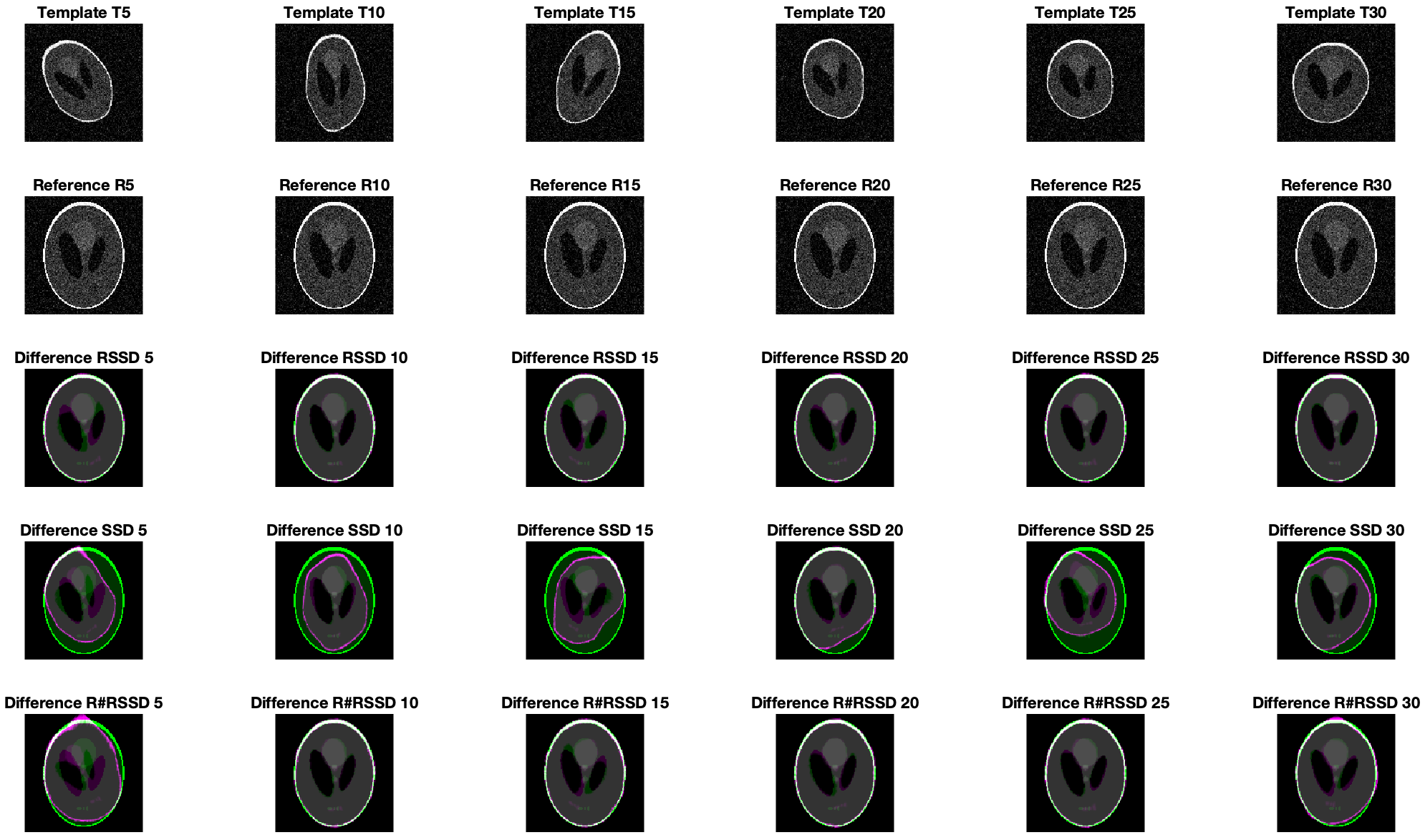}
	\caption{{\textbf{Sampling of images included in Test B  - high noise:} 
The first row displays the template T images from a sample of thirty high-noise images. The second row displays their respective six reference images with different seed noise.
The third, fourth, and fifth rows displays the difference between the reference image R and the registration image from the RSSD, SSD, and R$^\#$SSD methods, respectively, with the noise removed for comprehension.
		 }}
	\label{fig:SampleNoise10}
\end{figure}

\subsection{Test B: Registration of a non-affine deformation in images with random noise}

The objective of this experiment was to assess the performance of the three methods when registering thirty non-affine deformations on images with varying levels of noise. In this experiment, the same thirty R and T images utilized in the preceding experiment were considered. A random noise seed was generated for each, and they were added together to create thirty noisy images.
Examples of these noisy images are provided in Figures \ref{fig:SampleNoise5} and \ref{fig:SampleNoise10}. The registrations of the three methods were calculated using the coarse mesh.

The regularization parameter values used for this test were identical to those that produced the best results for each method in Test A. 
The results of the registrations obtained between noisy images in this test are presented in Figure \ref{fig:noisedRMSE}. 
Given that noise in images typically affects the similarity measures, such as the $RMSE(R,T_u)$, the deformation fields derived from the registrations were applied to the noise-free base images to obtain similarity measures comparable to those obtained previously.

The top graph of the Figure \ref{fig:noisedRMSE} illustrate the $RMSE(R,T_u)$ values obtained by the three methods when comparing the registered image with the reference image.
The graph demonstrates that the RSSD method, subsequently followed by  R$^\#$SSD, yields significantly lower and more homogeneous values for the registrations on images with and without noise than the other two methods. 
The low norm values between the real and estimated deformation fields in the bottom graph also serve to corroborate this conclusion.

\begin{figure}[t]
	\centering
	\includegraphics[width=0.99\linewidth]{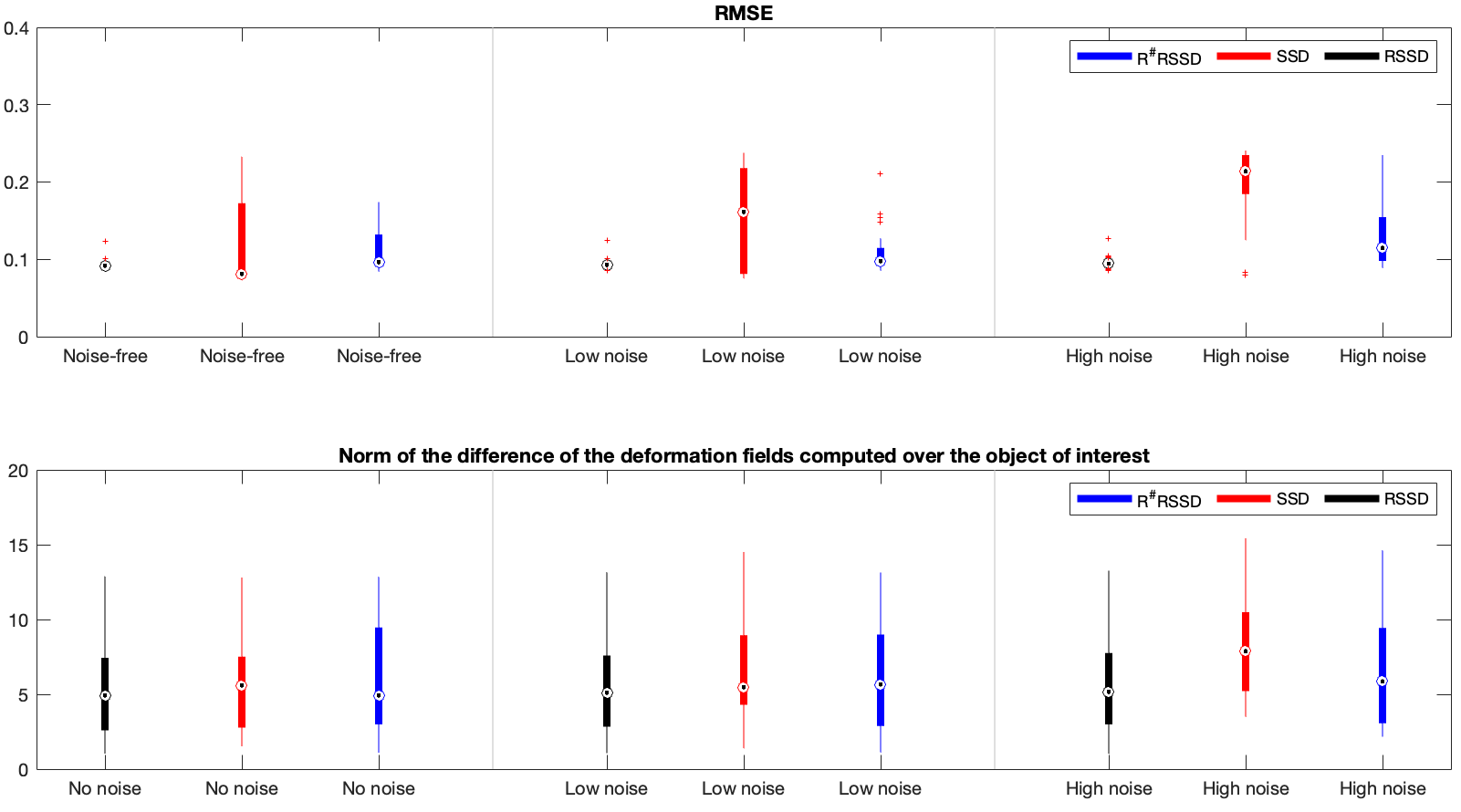}
	\caption[]{
\textbf{Results of the registration of 30 deformation images computed on noisy images.}
The distribution of the 30 values of the $RMSE(R,T_u)$ is illustrated in the top graph. 
The distribution of 30 difference norms of the difference of the deformation fields computed over the object of interest is presented in the bottom graph.}
\label{fig:noisedRMSE}
\end{figure}

\subsection{Test  C:  Convergence rates in the non-affine registration of noise-free images}

\begin{figure}[h]
	\centering
	\includegraphics[width=0.99\linewidth]{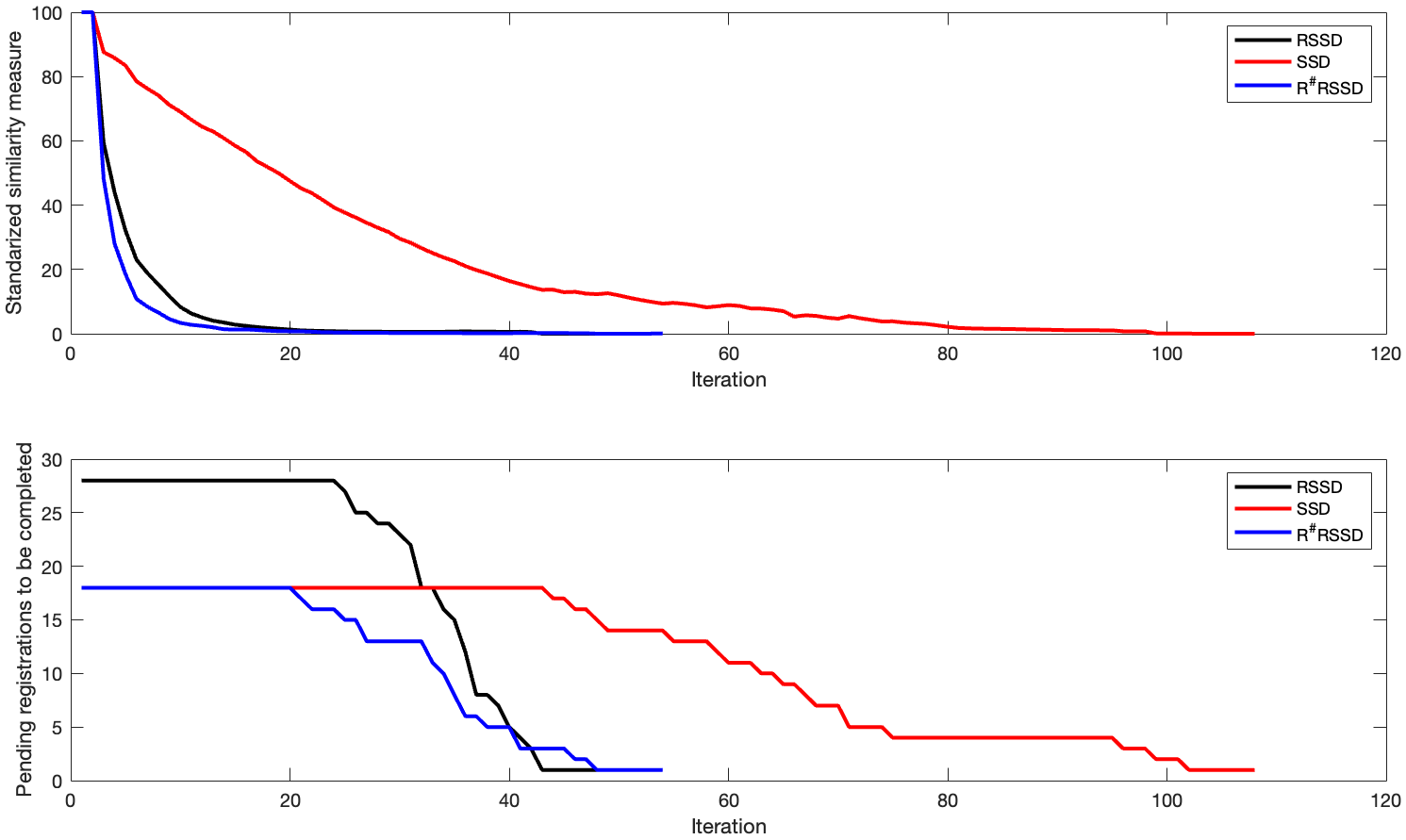}
	\caption{The top graph illustrates the convergence speed of the non-affine registration of noise-free images, whereas the bottom graph  illustrates the number of registrations pending completion in each iteration.}
	\label{fig:ConvRat_Noise-free}
\end{figure}

In this test, the convergence rates of the three methods were computed based on the results of Test A.
To ensure a valid comparison between the convergence rates of the three methods, the values provided by the similarity measures RSSD, SSD, and R$^\#$RSSD in each registration and iteration were normalized. 
 
In Figure \ref{fig:ConvRat_Noise-free} the top graph the mean values of the similarity measures of successful registrations (registrations with $RMSE(R,T_u)<0.1.$) over the course of each iteration. 
The unsuccessful results were not considered because they reached a local minimum and concluded prematurely, exhibiting a higher RMSE value.
The graph illustrates the exponential convergence of the RSSD and R$^\#$RSSD methods in comparison to the polynomial convergence of the SSD method. Once more, the RSSD and R$^\#$RSSD methods complete the required calculations in a smaller number of iterations.

The number of pending registration processes, beginning with the number of successful registrations, can be observed in the bottom graph. The graph illustrates that the RSSD method begins with 28 pending processes, while the SSD and R$^\#$RSSD methods begin with 18.  It can be observed that the totality of processes, both RSSD and R$^\#$RSSD, complete their iterations in less than 60 instances, whereas SSD requires a value greater than 100 iterations.
The graphs illustrate that the Radon transform-based methods reduce the solution exponentially, with this reduction occurring in fewer iterations, with the majority of registrations completed in less than 40 iterations.

\subsection{Test  D:  Convergence rates in the non-affine registration of noisy images}
 
The results of this test are analogous to those of Test C, but in this case, they are considered in the context of the experiments performed in Test B, namely, non-affine random deformations with low and high Gaussian noise. Once more, only successful registrations were considered for this experiment,   with $RMSE(R,T_u)<0.1$. 

In the case of low noise with variance equal to $0.05^2$, the upper plot of Figure \ref{fig:convratnoise-5} illustrates the convergence rates of the three methods. Additionally, the number of remaining processes is indicated at the bottom of the page.
As can be observed, the number of successful processes for the RSSD, SSD, and R$^\#$RSSD methods were 26, 12, and 18, respectively. The convergence speed was again exponential for the RSSD and R$^\#$RSSD methods, finishing most of them in less than 40 iterations. In contrast, the convergence speed for SSD was also polynomial, finishing most of them in over 60 iterations.

In the case of high noise with  $Var=0.1^2$, the upper graph of Figure \ref{fig:convratnoise-10} illustrates the convergence speeds of the three methods. The lower graph displays the number of processes that were completed.
As previously observed, the number of successful processes for the RSSD, SSD, and R$^\#$RSSD methods was 26, 2, and 10, respectively. The convergence speed was again exponential for the RSSD and R$^\#$RSSD methods, with the majority of processes completed in less than 40 iterations.  In contrast, the convergence speed for the two successful cases of the SSD method was also polynomial, with both processes completed in over 80 iterations.
In the latter case, it can be observed that the RSSD method exhibits a superior performance in registering images with high noise, in comparison to the 6.7\% success rate observed for the SSD method.

\begin{figure}[h]
	\centering
	\includegraphics[width=0.99\linewidth]{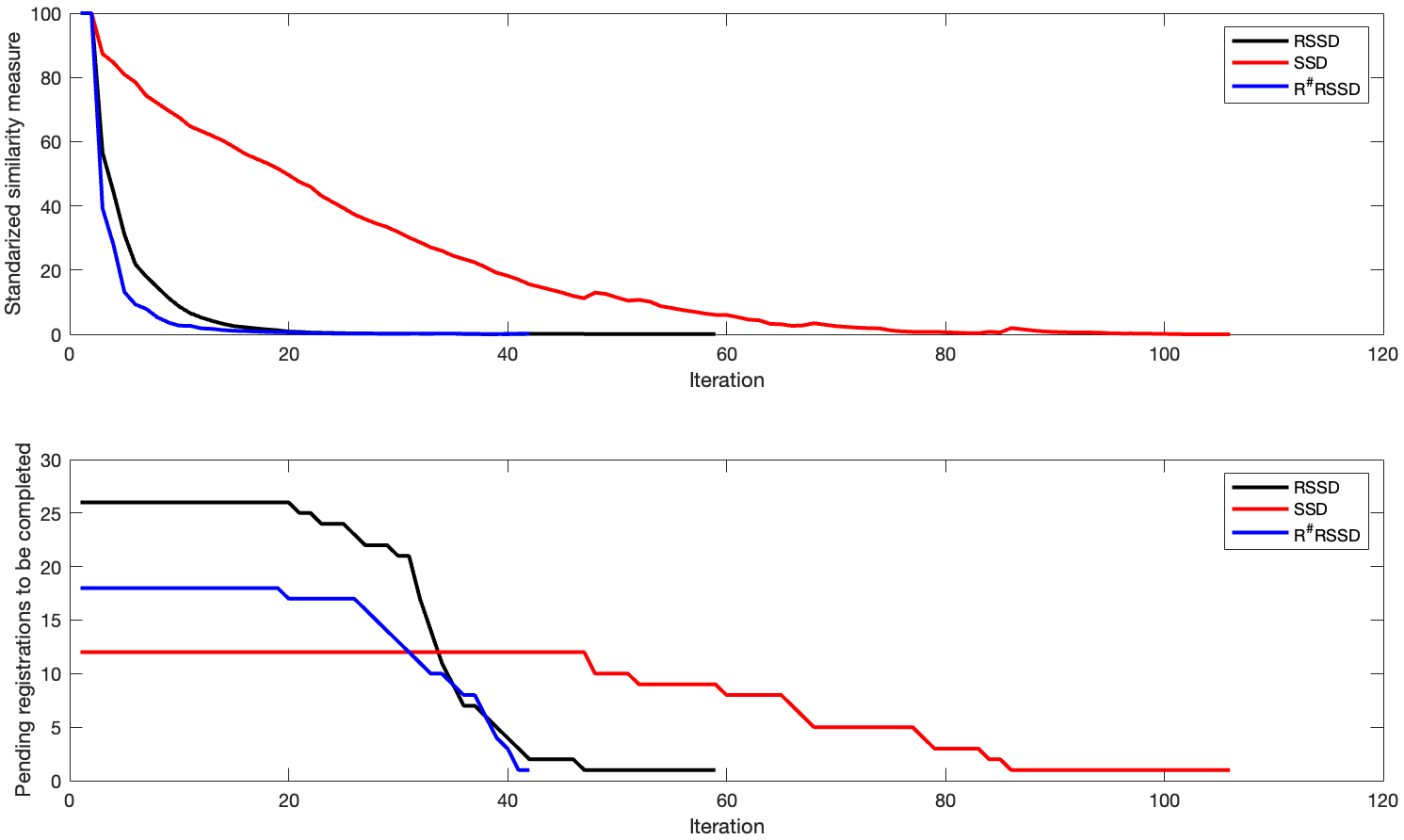}
	\caption{Convergence speed in the top graph and registrations pending completion  in the bottom graph, in the non-affine registration of images  with low noise}
	\label{fig:convratnoise-5}
\end{figure}

\newpage
\begin{figure}[h!]
	\centering
	\includegraphics[width=0.99\linewidth]{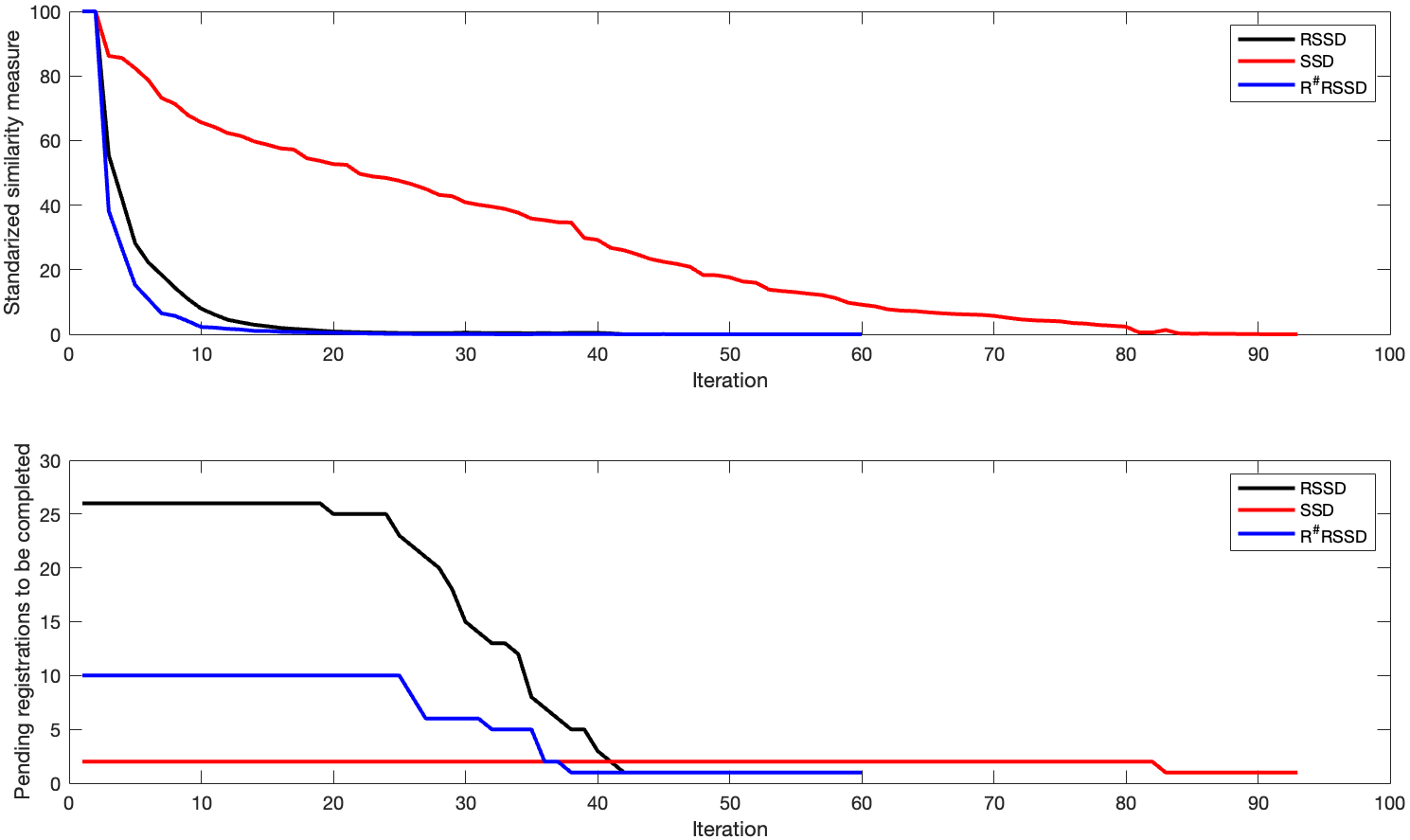}
	\caption{Convergence speed in the top graph and registrations pending completion  in the bottom graph, in the non-affine registration of images  with high noise}
	\label{fig:convratnoise-10}
\end{figure}

\begin{figure}[t!]
	\centering
	\includegraphics[width=0.99\linewidth]{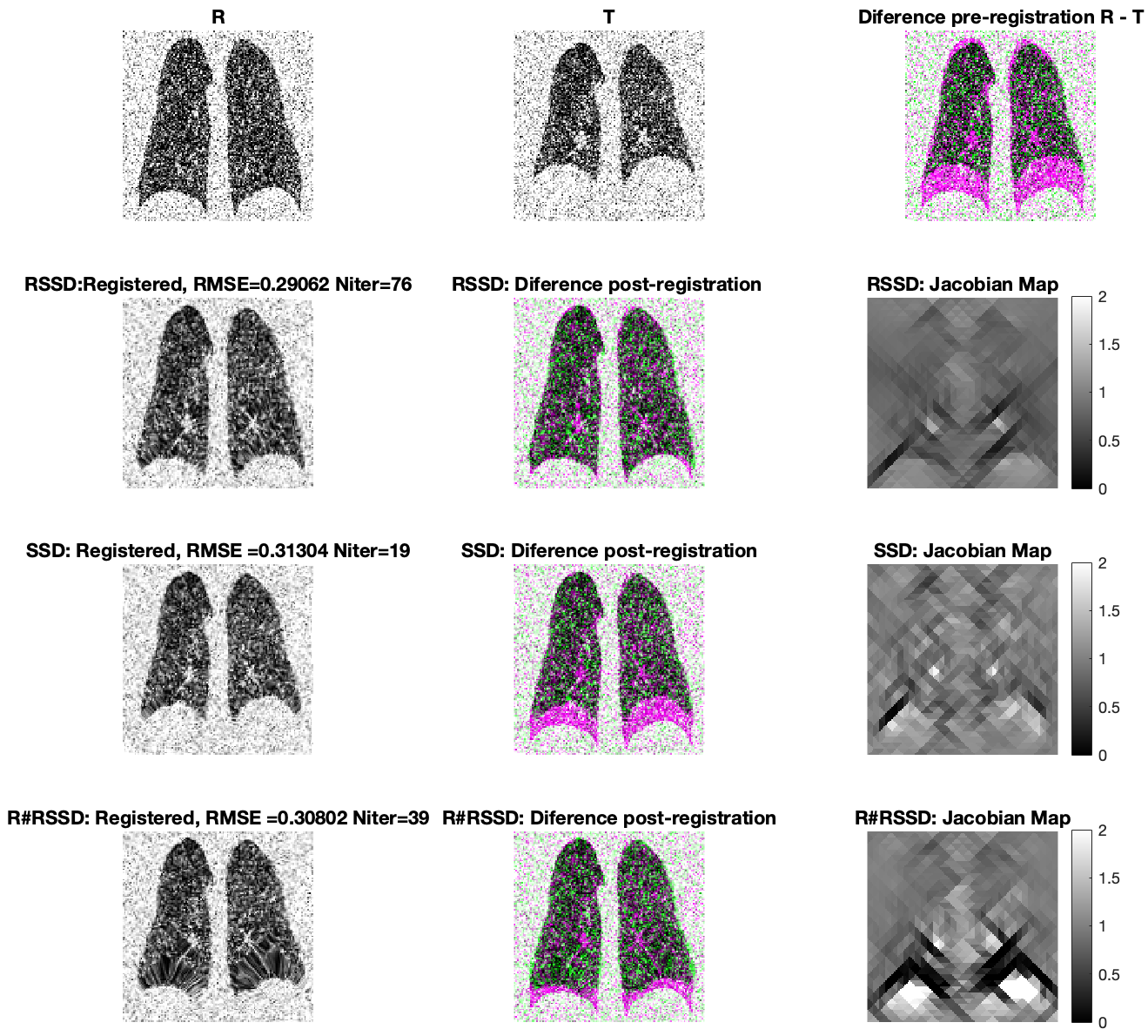}
	\caption{Results of image registration of noisy lung images. The first row shows the reference R image, the template T image, and the initial difference between them.
The second row shows the results provided by the RSSD method, the third row shows the results provided by the SSD method, and the fourth row shows the results provided by the R$^\#$RSSD method.}
	\label{fig:lungaplication}
\end{figure}

\subsection{Example E: Application to lung deformations}
This example shows the initial motivation for this work, which was to improve existing lung image registration methods.
Figure \ref{fig:lungaplication} illustrates the progression of images from left to right. The first row displays the reference R image, the template T image, and the initial R-T difference image, which has been subjected to the addition of high Gaussian noise, as is typical of real medical images. We are grateful to the UC Computational Medicine Lab for providing these images.
The results of the three registration methods are presented in the second, third, and fourth rows, respectively. The second row displays the results of the RSSD method, the third row shows the SSD method, and the fourth row presents the R$^\#$RSSD method. 
The images in the first column displays the registrations with the corresponding  $RMSE(R,T_u)$ result. The images in the second column illustrate the differences between the registered images and the reference image R. The images in the third column display the Jacobian maps on the triangulations, with color indicating the variations in area for each triangle in the post-registration domain.
In this experiment, the registrations of the three methods were calculated using the fine mesh with the same regularization parameters. This was necessary to ensure that the deformations could be captured with sufficient detail.

It can be observed that in this case of lung imaging with high noise, the RSSD method was the only one that achieved satisfactory registration. 
Although the RMSE values are similar among the three methods, it is demonstrated that, with the difference image, the RSSD method was the only one to achieve an appropiate registration. 
More precisely, the Jacobian map indicates that the lung exhibited the expected deformations in the intended areas.

\section{Discussion}
\label{sec:Discussion}

Previous studies have evidenced the widespread application of intensity-based similarity measures in DIR methods due to their simple implementation and their efficacy in capturing small affine deformations. The SSD similarity measure exemplifies this phenomenon \cite{ashburner1999nonlinear,brown1992survey,capek1999optimisation,modersitzki2004numerical,wachinger2012simultaneous}. 
The aforementioned methods are not without defects. They are unsuitable for capturing large deformations and require similar or correlated intensities in the images for successful registration. Furthermore, the potential for false matches in the presence of noise must be considered \cite{BAJCSY19891,zitova2003image}.

In recent decades, registration methods that incorporate some form of projection into their formulation have gained popularity due to their robustness in registering noisy images \cite{albu2014transformed,BAJCSY19891,cain2001projection,khamene2006novel,sauer1996efficient}. In particular, methods that employ the Radon transform have demonstrated efficacy in accurately and rapidly computing global affine deformations \cite{mooser2013estimation,mooser2009estimation,nacereddine2015similarity,yan2005edge}.
The formulations of these methods are based on the linearity and homogeneity properties of the Radon transform, which limits their efficacy in capturing locally affine or non-affine deformations. The potential for extensions to these cases has yet to be fully explored.

In this study, two similarity measures based on the Radon transform and its inverse were proposed, designated as $RSSD$ and $R^\#RSSD$, respectively. Both methods were configured with the LEE regularizer and a deformation model analogous to linear finite elements was incorporated. Additionally, a Quasi-Newton optimization method was utilized. As a contrasting model, a method utilizing the SSD similarity measure and the remaining configurations employed in the proposed methods was considered.

The results demonstrated the existence and uniqueness of the solution to the proposed DIR problems. Furthermore, experimental analysis indicates that the proposed methods exhibit lower and more homogeneous errors than the contrast method when registering non-affine deformations between noisy and non-noisy images. 
In terms of convergence, the proposed methods exhibited an exponential rate of convergence, reaching the solution in a reduced number of iterations compared to the SSD contrast approach. As documented in \cite{wachinger2012simultaneous}, the SSD method also exhibited a polynomial trend in convergence. Furthermore, when registering a lung medical imaging case with high levels of noise and local non-affine deformations, the proposed methods demonstrated superior performance.
The results presented in this study are innovative in that they extend the simplicity of intensity-based methods and the robustness against noise of projections-based methods to the development of new registration methods that enable the identification of both global and local non-affine deformations between images, even in cases of highly noisy images. 
Another interesting aspect of this work is that it not only presents a well-performing method, but also provides conditions that guarantee the existence and uniqueness of the solutions reached.
From a practical standpoint, these advances present a significant opportunity to enhance image matching techniques across various fields of knowledge, including the diagnosis of diseases related to lung tissue stiffness. In this context, the registration of non-affine deformations on noisy images represents a pivotal step in computing tissue stresses \cite{concha2018micromechanical,Cruces2019mapping,hurtado2017spatial,retamal2018does}.

It is important to note that the proposed methods present certain limitations related to their configurations and implementation processes. One such limitation is that the procedure for choosing the regularization constant of the method is obtained heuristically and manually incorporated into the program. Furthermore, the numerical implementation of the proposed methods may be costly due to their reliance on images in their original size, uniform meshes throughout the domain, and explicit gradients calculated based on centered finite differences in the optimization method.
From this, opportunities for improving the proposed algorithms arise. These include the automation of the regularization constant search process, the implementation of a pyramidal approach for better performance, the incorporation of subsampling in the projections used, the automatic generation of adaptive meshes to achieve localized refinement in areas of finer details, the development of an optimized method for computing the explicit gradient, and the adaptation of the proposed methods for the case of non-affine registration in 3D.

\section*{Acknowledgments}
Daniel Hurtado thanks the funding of ANID-Fondecyt Regular \#1220465.
Axel Osses thanks the funding of ANID-Fondecyt Regular \#1201311, CMM FB210005 Basal, FONDAP/15110009, ACIPDE MATH190008, Millennium Program NCN19-161and Data Observatory ANID Technology Center DO210001. \newline Rodrigo Quezada thanks the funding of ANID BECAS/DOCTORADO NACIONAL/ 2016-21160751.

\section*{References}
\bibliography{references}

\providecommand{\newblock}{}
\begin{thebibliography}{10}
\expandafter\ifx\csname url\endcsname\relax
  \def\url#1{{\tt #1}}\fi
\expandafter\ifx\csname urlprefix\endcsname\relax\def\urlprefix{URL }\fi
\providecommand{\eprint}[2][]{\url{#2}}

\bibitem{ruthotto2015non}
Ruthotto L and Modersitzki J 2015 {\em Handbook of Mathematical Methods in
  Imaging\/}  2005--2051

\bibitem{oliveira2014medical}
Oliveira F~P and Tavares J~M~R 2014 {\em Computer Methods in Biomechanics and
  Biomedical Engineering\/} {\bf 17} 73--93 pMID: 22435355

\bibitem{concha2018micromechanical}
Concha F, Sarabia-Vallejos M and Hurtado D~E 2018 {\em Journal of the Mechanics
  and Physics of Solids\/} {\bf 112} 126--144

\bibitem{Cruces2019mapping}
Cruces P, Erranz B, Lillo F, Sarabia-Vallejos M~A, Iturrieta P, Morales F,
  Blaha K, Medina T, Diaz F and Hurtado D~E 2019 {\em BMJ Open Respiratory
  Research\/} {\bf 6}

\bibitem{hurtado2017spatial}
Hurtado D~E, Villarroel N, Andrade C, Retamal J, Bugedo G and Bruhn A 2017 {\em
  Biomechanics and Modeling in Mechanobiology\/} {\bf 16} 1413--1423

\bibitem{retamal2018does}
Retamal J, Hurtado D, Villarroel N, Bruhn A, Bugedo G, Amato M~B~P, Costa
  E~L~V, Hedenstierna G, Larsson A and Borges J~B 2018 {\em Critical care
  medicine\/} {\bf 46} e591--e599

\bibitem{aubert2006mathematical}
Aubert G and Kornprobst P 2006 {\em Mathematical problems in image processing:
  partial differential equations and the calculus of variations\/} vol 147
  (Springer Science \& Business Media)

\bibitem{vese2016variational}
Vese L~A and Le~Guyader C 2016 {\em Variational methods in image processing\/}
  (CRC Press)

\bibitem{oektem2017shape}
Oektem O, Chen C, Domanic N~O, Ravikumar P and Bajaj C 2017 {\em Inverse
  problems\/} {\bf 33} 035004

\bibitem{barnafi2018primal}
Barnafi N, Gatica G~N and Hurtado D~E 2018 {\em SIAM Journal on Imaging
  Sciences\/} {\bf 11} 2529--2567

\bibitem{chen2018indirect}
Chen C and Oktem O 2018 {\em SIAM Journal on Imaging Sciences\/} {\bf 11}
  575--617

\bibitem{maintz1998survey}
Maintz J~A and Viergever M~A 1998 {\em Medical image analysis\/} {\bf 2} 1--36

\bibitem{BAJCSY19891}
Bajcsy R and Kovačič S 1989 {\em Computer Vision, Graphics, and Image
  Processing\/} {\bf 46} 1--21 ISSN 0734-189X
  \urlprefix\url{https://www.sciencedirect.com/science/article/pii/S0734189X89800143}

\bibitem{ashburner1999nonlinear}
Ashburner J and Friston K~J 1999 {\em Human brain mapping\/} {\bf 7} 254--266

\bibitem{brown1992survey}
Brown L~G 1992 {\em ACM computing surveys (CSUR)\/} {\bf 24} 325--376

\bibitem{capek1999optimisation}
Capek K 1999 Optimisation strategies applied to global similarity based image
  registration methods {\em International Conferences in Central Europe on
  Computer Graphics, Visualization and Computer Vision (WSCG)\/} vol~2 pp
  369--374

\bibitem{modersitzki2004numerical}
Modersitzki J 2004 {\em Numerical methods for image registration\/} (Oxford
  University Press on Demand)

\bibitem{cain2001projection}
Cain S~C, Hayat M~M and Armstrong E~E 2001 {\em IEEE transactions on image
  processing\/} {\bf 10} 1860--1872

\bibitem{albu2014transformed}
Albu F and Corcoran P 2014 Transformed integral projection method for global
  alignment of second order radially distorted images {\em 2014 Signal
  Processing: Algorithms, Architectures, Arrangements, and Applications
  (SPA)\/} (IEEE) pp 42--47

\bibitem{khamene2006novel}
Khamene A, Chisu R, Wein W, Navab N and Sauer F 2006 A novel projection based
  approach for medical image registration {\em International Workshop on
  Biomedical Image Registration\/} (Springer) pp 247--256

\bibitem{sauer1996efficient}
Sauer K and Schwartz B 1996 {\em IEEE Transactions on circuits and systems for
  video technology\/} {\bf 6} 513--518

\bibitem{mooser2013estimation}
Mooser R, Forsberg F, Hack E, Sz{\'e}kely G and Sennhauser U 2013 {\em Machine
  vision and applications\/} {\bf 24} 419--434

\bibitem{mooser2009estimation}
Mooser R, Hack E, Sennhauser U and Sz{\'e}kely G 2009 Estimation of affine
  transformations directly from tomographic projections {\em 2009 Proceedings
  of 6th International Symposium on Image and Signal Processing and Analysis\/}
  (IEEE) pp 377--382

\bibitem{nacereddine2015similarity}
Nacereddine N, Tabbone S and Ziou D 2015 {\em Pattern Recognition\/} {\bf 48}
  2227--2240

\bibitem{JiangshengYou1998}
{Jiangsheng You}, {Weiguo Lu}, {Jian Li}, {Gene Gindi} and {Zhengrong Liang}
  1998 {\em Proceedings 1998 International Conference on Image Processing.
  ICIP98 (Cat. No.98CB36269)\/} {\bf 1} 847--851
  \urlprefix\url{http://ieeexplore.ieee.org/document/723649/}

\bibitem{Mao2007a}
Mao W, Li T, Wink N and Xing L 2007 {\em Medical Physics\/} {\bf 34} 3596--3602
  ISSN 00942405 \urlprefix\url{http://doi.wiley.com/10.1118/1.2767402}

\bibitem{yan2005edge}
Yan H, Liu J and Sun J 2005 Edge projection-based image registration {\em
  Knowledge-Based Intelligent Information and Engineering Systems: 9th
  International Conference, KES 2005, Melbourne, Australia, September 14-16,
  2005, Proceedings, Part II 9\/} (Springer) pp 1231--1237

\bibitem{duvant2012inequalities}
Duvant G and Lions J~L 2012 {\em Inequalities in mechanics and physics\/} vol
  219 (Springer Science \& Business Media)

\bibitem{hughes2012finite}
Hughes T~J 2012 {\em The finite element method: linear static and dynamic
  finite element analysis\/} (Courier Corporation)

\bibitem{wachinger2012simultaneous}
Wachinger C and Navab N 2012 {\em IEEE transactions on pattern analysis and
  machine intelligence\/} {\bf 35} 1221--1233

\bibitem{zitova2003image}
Zitova B and Flusser J 2003 {\em Image and vision computing\/} {\bf 21}
  977--1000

\end{thebibliography}

\end{document}